\documentclass{article}
\usepackage[pdftex]{graphicx}
\usepackage{graphicx}
\usepackage{amssymb}
\usepackage{subcaption}
\usepackage{caption}
\usepackage{wrapfig}
\usepackage{amsmath}
\usepackage{amsthm} 

\usepackage[utf8]{inputenc} 
\usepackage[T1]{fontenc}    
\usepackage{hyperref}       
\usepackage{url}            
\usepackage{booktabs}       
\usepackage{amsfonts}       
\usepackage{nicefrac}       
\usepackage{microtype}      
\usepackage{xcolor}         
\usepackage{amsthm}
\newtheorem{corollary}{Corollary}
\newtheorem{proposition}{Proposition}
\usepackage[preprint]{neurips_2026}

\usepackage[utf8]{inputenc} 
\usepackage[T1]{fontenc}    
\usepackage{hyperref}       
\usepackage{url}            
\usepackage{booktabs}       
\usepackage{amsfonts}       
\usepackage{nicefrac}       
\usepackage{microtype}      
\usepackage{xcolor}         
\usepackage{multirow}
\usepackage{booktabs}
\title{MPD$^2$-Router:  Mask-aware Multi-expert Prior-regularized Dual-head Deferral Router in Glaucoma Screening and Diagnosis}

\author{
Wenxin Zhan\\
Baruch college\\
\texttt{wenxin.zhan@baruchmail.cuny.edu}
}
\begin{document}
\maketitle
\begin{abstract}
Learning-to-defer (L2D) can make glaucoma screening safer by routing difficult/uncertain cases to humans, yet standard formulations overlook expert availability, heterogeneous readers behavior, workload imbalance, asymmetric diagnostic harm, case difficulty from morphology and deployment shift. We introduce MPD$^2$-Router, a mask-aware multi-expert deferral framework that recasts ophthalmic triage as constrained human--AI routing: whether to defer and to which available expert. It couples a dual-head deferral/allocation policy with mask-aware Gumbel--sigmoid gating that strictly enforces per-sample availability, and fuses uncertainty, morphology, image-quality, and OOD signals. Training uses an asymmetric cost-sensitive objective with an augmented-Lagrangian deferral budget, a group-specific distribution prior, and a rank-majorization JS regularizer that jointly prevent expert collapse without forcing uniform allocation. Across three cross-national glaucoma cohorts (REFUGE, CHAKSU, ORIGA) with a frozen REFUGE-trained backbone, MPD$^2$-Router substantially lowers clinical cost and improves MCC over AI-only at a moderate deferral rate. It is Pareto-optimal in F1--MCC--cost, robust under cross-domain shift, and yields balanced expert utilization. 
\end{abstract}
\vspace{-2mm}
\section{Introduction}
\vspace{-2mm}
 Artificial intelligence has achieved strong performance on ophthalmic image classification, yet the dominant formulation of the problem remains clinically incomplete and unsafe. Most existing systems are optimized as one-shot predictors: given an image, output a label. Real ophthalmic practice does not operate this way, it's instead a \textit{sequential decision process} under uncertainty, where clinicians must determine whether the available evidence is sufficient, whether additional imaging or testing is required, and whether a case should be escalated to another specialist. In other words, the clinically relevant question is not merely  \emph{what is the label}, but \emph{what is the appropriate next action}. This mismatch between benchmark formulation and clinical workflow directly limits the safe integration of AI in ophthalmology. 

Our study is motivated by two clinical realities often abstracted away in
standard machine-learning benchmarks. First, human experts are fallible and heterogeneous rather than interchangeable oracles. Ophthalmologists, retina specialists, graders, and trainees differ in experience, subspecialty knowledge, decision thresholds, and tolerance for ambiguity; inter-grader variability can directly affect referral decisions.\citep{krause2018grader, teoh2023variability}. Treating human review as a single noiseless reference, as much of the L2D literature implicitly does,  discards precisely the structure that a safe routing system must model. Second, AI systems are not uniformly reliable. Confidence alone is a fragile triage signal: difficult morphology, low image quality, and out-of-distribution (OOD) inputs can significantly undermine calibration and accuracy. Together, these facts reframe glaucoma screening as a routing problem rather than a classification problem: \emph{when should the system act autonomously, when should it defer, and to whom?}

L2D formalizes part of this problem by allowing a model to pass selected cases to a human decision-maker \citep{madras2018predict, mozannar2020consistent}, with recent extensions to calibrated, team-based, and multi-expert settings \citep{verma2022calibrated, mao2023two}.  However, existing formulations are largely built around generic and naive deployment assumptions, such as always-available human, an expert softmax with limited treatment of clinical cost and capacity. Glaucoma triage requires a richer formulation that accounts for per-sample expert feasibility, asymmetric diagnostic harm, morphology- and quality-dependent case
difficulty, distribution shift, and the risk of concentrating deferred cases on scarce specialists. A useful glaucoma triage system should therefore be selective rather than indiscriminate, availability-aware rather than static, and Pareto-favorable: improving team-level utility without simply routing all hard cases to the same scarce expert.

\textbf{Our contributions.}
We formulate glaucoma diagnosis as a constrained human--AI routing problem and
make four contributions. First, we introduce a multi-signal deferral policy that
conditions jointly on expert availability, AI uncertainty, optic-disc/cup
morphology, image-quality risk, and OOD indicators. Second, we propose a
dual-head router that separates \emph{whether} to defer from \emph{to whom} to
defer, using mask-aware Gumbel--sigmoid support selection and masked expert
allocation to enforce feasibility. Third, we regularize allocation with a
group-specific empirical-Bayes-style competence prior and a rank-majorization JS
penalty, mitigating expert collapse without forcing uniform routing. Fourth, we
train under an asymmetric cost-sensitive objective with an augmented-Lagrangian
deferral budget, yielding clinically aligned and workload-aware decisions.
Together, MPD$^2$-Router provides a clinically grounded framework for
analyzing when autonomous AI is appropriate, when deferral is necessary, and
how scarce expert attention should be allocated safely under shift.
\vspace{-2mm}
\section{Preliminaries}
\label{sec:preliminaries}
\vspace{-2mm}
We study cost-sensitive ophthalmic triage with a frozen diagnostic predictor and a trainable multi-expert routing policy. Let \(\mathcal Y=\{0,1\}\), where \(y=1\) denotes glaucoma and \(y=0\) denotes non-glaucoma. Let \(\mathcal E=\{1,\ldots,M\}\) be the set of human experts. Action \(a=0\) retains the frozen AI prediction, while \(a=j\in\mathcal E\) defers the case to expert \(j\).

The routing training set is \(\mathcal D=\{(x_i,y_i,\ell_i,\hat y_{i,1:M},m_i^{\mathrm{exp}})\}_{i=1}^{N}\), where \(x_i\in\mathcal X\), \(y_i\in\{0,1\}\), \(\ell_i=(\ell_i^{(0)},\ell_i^{(1)})\in\mathbb R^2\) are logits from a pretrained frozen AI classifier, and \(p_i^{\mathrm{ai}}=\Pr_{\mathrm{AI}}(y_i=1\mid x_i)=\operatorname{softmax}(\ell_i)_1\). The expert label \(\hat y_{i,j}\in\{0,1,\mathrm{NA}\}\) is observed only when expert \(j\) is feasible for sample \(i\), encoded by \(m_{i,j}^{\mathrm{exp}}\in\{0,1\}\), with \(m_{i,j}^{\mathrm{exp}}=1\Longleftrightarrow j\) is feasible for sample \(i\). The feasible expert set and feasible action set are \(S_i=\{j\in\mathcal E:m_{i,j}^{\mathrm{exp}}=1\}\) and \(\mathcal A_i=\{0\}\cup S_i\). Action \(0\) is always feasible; if \(S_i=\varnothing\), we use the convention that \(d_i=0\) and the router deterministically selects action \(0\).

A stochastic router assigns \(\pi_i=\pi_\theta(\cdot\mid x_i,m_i^{\mathrm{exp}})\in\Delta(\mathcal A_i)\), where the masked simplex is
\[
\Delta(\mathcal A_i)=\{\pi\in\mathbb R_+^{M+1}:\sum_{a=0}^{M}\pi_a=1,\ \pi_a=0\ \forall a\notin\mathcal A_i\}.
\]
For \(S_i\neq\varnothing\), any feasible policy admits the defer--allocate factorization \(\pi_{i,0}=1-d_i\) and \(\pi_{i,j}=d_iq_{i,j}\), where \(d_i\in[0,1]\) and
\[
q_i\in\Delta(S_i):=\{q\in\mathbb R_+^M:\sum_{j\in S_i}q_j=1,\ q_j=0\ \forall j\notin S_i\}.
\]
Here \(d_i\) is the soft deferral mass and \(q_i\) is the conditional allocation distribution over feasible experts. Conversely, when \(d_i>0\), \(q_{i,j}=\pi_{i,j}/\sum_{k=1}^{M}\pi_{i,k}\) and \(d_i=1-\pi_{i,0}=\sum_{j=1}^{M}\pi_{i,j}\).

Soft routing quantities are optimization-facing. The empirical soft deferral rate is \(\bar d=N^{-1}\sum_{i=1}^{N}d_i\), and the expert usage masses are \(P_j^{\mathrm{soft}}=N^{-1}\sum_{i=1}^{N}\pi_{i,j}=N^{-1}\sum_{i=1}^{N}d_iq_{i,j}\). At deployment, the stochastic policy is converted into the deterministic routing decision \(\hat a_i=\arg\max_{a\in\mathcal A_i}\pi_{i,a}\), with ties broken in favor of the AI action. This induces the hard deferral indicator \(d_i^{\mathrm{hard}}=\mathbf 1[\hat a_i\neq 0]\) and hard expert frequency \(f_j^{\mathrm{hard}}=N^{-1}\sum_{i=1}^{N}\mathbf 1[\hat a_i=j]\). Thus \(d_i\) and \(\bar d\) are differentiable training quantities, whereas \(d_i^{\mathrm{hard}}\) and \(f_j^{\mathrm{hard}}\) describe realized deployment behavior.

We use an asymmetric clinical loss. For the frozen AI branch, \(C_i^{\mathrm{ai}}=c_{\mathrm{fn}}\,y_i(1-p_i^{\mathrm{ai}})+c_{\mathrm{fp}}\,(1-y_i)p_i^{\mathrm{ai}}\), with \(c_{\mathrm{fn}}>c_{\mathrm{fp}}>0\). For feasible expert \(j\in S_i\), \(C_{i,j}^{\mathrm{exp}}=c_{\mathrm{fn}}\mathbf 1[y_i=1,\hat y_{i,j}=0]+c_{\mathrm{fp}}\mathbf 1[y_i=0,\hat y_{i,j}=1]\). Let \(\kappa_j\ge0\) be the operational cost of expert \(j\), with \(\kappa_0=0\). For \(\gamma\ge0\), the expected per-sample routing loss is
\[
\vspace{-2mm}
L_i(\pi_i)=\pi_{i,0}C_i^{\mathrm{ai}}+\sum_{j=1}^{M}\pi_{i,j}(C_{i,j}^{\mathrm{exp}}+\gamma\kappa_j).
\vspace{-2mm}
\]
Under the defer--allocate factorization,
\[
L_i(d_i,q_i)=(1-d_i)C_i^{\mathrm{ai}}+d_i\sum_{j=1}^{M}q_{i,j}(C_{i,j}^{\mathrm{exp}}+\gamma\kappa_j)=C_i^{\mathrm{ai}}+d_i\Delta_i(q_i),
\]
where \(\Delta_i(q_i)=\sum_{j=1}^{M}q_{i,j}(C_{i,j}^{\mathrm{exp}}+\gamma\kappa_j)-C_i^{\mathrm{ai}}\) is the marginal cost of replacing the AI decision by a deferred expert decision. The empirical clinical--operational risk is \(\mathcal L_{\mathrm{cost}}(\theta)=N^{-1}\sum_{i=1}^{N}L_i(d_i,q_i)\). 
The learning objective is to minimize this cost-sensitive routing risk under sample-dependent feasibility masks and a soft deferral budget \(\bar d\le\rho_{\mathrm{def}}\). The next section specifies the mask-aware policy class and the regularizers used to stabilize multi-expert allocation under heterogeneous availability.
\vspace{-2mm}
\section{Method}
\label{sec:method}
\vspace{-2mm}
We propose \textbf{MPD\(^2\)-Router}, a mask-aware, prior-regularized dual-head router with policy
\[
\pi_\theta(x_i,m_i^{\mathrm{exp}})
=
\bigl(1-d_i,\;d_iq_i\bigr),
\]
where \(d_i\in(0,1)\) denotes the probability of deferring case \(i\) to a
human expert, and \(q_i\in\Delta^{M-1}\) is a mask-aware conditional allocation
distribution over feasible experts. Figure~\ref{fig:mpd2_workflow_app} in Appendix~\ref{app:workflow} provides the complete workflow schematic of MPD\(^2\)-Router.
\vspace{-2mm}
\subsection{Router representation}
\vspace{-2mm}
\label{subsec:router_representation}
The router state combines three sources of evidence: frozen AI logits, risk
signals, and morphology-derived structural features. Let
\[
P_i=\operatorname{softmax}(\ell_i),
\qquad
p_i^{\mathrm{ai}}=P_{i,1},
\qquad
u_i=1-\max_{c\in\{0,1\}}P_{i,c}.
\]
The risk vector is
\[
r_i^{\mathrm{risk}}
=
\left[
\mathrm{ViM}_i,\;
r_i^{\mathrm{qual}},\;
u_i
\right]
\in\mathbb R^3,
\]
where \(\mathrm{ViM}_i\) is the standardized Virtual-logit Matching score
\citep{wang2022vim}, \(r_i^{\mathrm{qual}}\) is an image-quality risk score,
and \(u_i\) is maximum-probability uncertainty.

Let $\xi_i^{\mathrm{str}} =
\begin{bmatrix}
\mathrm{vCDR}_i\\
\mathrm{aCDR}_i
\end{bmatrix}
$
denote optic-disc/cup morphology. A scalar structural risk estimate is
\[
p_i^{\mathrm{str}}
=
\sigma
\left(
w_{\mathrm{str}}^\top \xi_i^{\mathrm{str}}
+
b_{\mathrm{str}}
\right),
\]
and the structural feature vector is
\[
r_i^{\mathrm{str}}
=
\left[
2p_i^{\mathrm{str}}-1,\;
\left|p_i^{\mathrm{ai}}-p_i^{\mathrm{str}}\right|
\right]
\in\mathbb R^2 .
\]
The first coordinate is a centered morphology margin; the second measures
AI--structure disagreement.

The branch encoders are
\[
h_i^{\mathrm{risk}}
=
\phi_{\mathrm{risk}}(r_i^{\mathrm{risk}}),
\qquad
h_i^{\mathrm{str}}
=
\phi_{\mathrm{str}}(r_i^{\mathrm{str}}),
\qquad
h_i^{\mathrm{ai}}
=
\phi_{\mathrm{ai}}(\ell_i),
\]
and the fused router representation is $
h_i
=
\phi_{\mathrm{fuse}}
\left(
h_i^{\mathrm{risk}}
\;\|\;
h_i^{\mathrm{str}}
\;\|\;
h_i^{\mathrm{ai}}
\right).$
\vspace{-1mm}
\subsection{Mask-aware dual-head policy}
\label{subsec:mask_aware_policy}
\vspace{-1mm}
The policy is factorized into two coupled decisions:
\[
\text{(i) whether to defer},\qquad
\text{(ii) how to allocate deferred mass across feasible experts}.
\]
The defer head maps \(h_i\) to a soft deferral mass
\[
d_i
=
\mathbf 1[|S_i|>0]\,
\sigma
\left(
f_{\mathrm{def}}(h_i)
\right)
\in[0,1].
\]
The factor \(\mathbf 1[|S_i|>0]\) enforces autonomous prediction when no expert
is feasible.

The expert head is first parameterized through a shared expert trunk $z_i=\varphi(h_i)$, it then produces two sets of logits:
\[
\gamma_i=W_gz_i+c_g\in\mathbb R^M,
\qquad
\beta_i=W_az_i+c_a\in\mathbb R^M ,
\]
where \(\gamma_i\) controls stochastic support selection and \(\beta_i\) controls allocation within the selected support.

\paragraph{Stage I: masked stochastic support selection.}

Feasibility is imposed before randomization by the extended-real masking map
\[
  \bar\gamma_{i,j} \;=\; \gamma_{i,j} + \log m_{i,j}^{\mathrm{exp}}
  \;=\;
  \begin{cases}
    \gamma_{i,j}, & j\in S_i,\\
    -\infty,      & j\notin S_i,
  \end{cases}
\]
Equivalently, the masked Bernoulli odds are
$
\zeta_{i,j}
=
\exp(\bar\gamma_{i,j})
=
m_{i,j}^{\mathrm{exp}}\exp(\gamma_{i,j}).
$
Using the Gumbel-max representation, let
\[
G_{i,j}^{(1)},G_{i,j}^{(0)}
\stackrel{\mathrm{i.i.d.}}{\sim}
\mathrm{Gumbel}(0,1),
\qquad
\eta_{i,j}=G_{i,j}^{(1)}-G_{i,j}^{(0)}
\sim\mathrm{Logistic}(0,1).
\]
Equivalently, in distribution,
\[
\eta_{i,j}
\overset{d}{=}
\log U_{i,j}-\log(1-U_{i,j}),
\qquad
U_{i,j}\sim\mathrm{Unif}(0,1).
\]
The hard Bernoulli gate can be written as
\begin{equation}
s_{i,j}^{\mathrm{hard}}
=
\mathbf 1
\!\left[
\bar\gamma_{i,j}+\eta_{i,j}\ge 0
\right].
\label{eq:hard_gate_clean}
\end{equation}

Replacing the discontinuous Bernoulli gate by a Binary Concrete
(Gumbel--sigmoid) relaxation gives
\begin{align}
\tilde s_{i,j}
&=
\sigma\!\left(
\frac{\bar\gamma_{i,j}+\eta_{i,j}}{\tau_g}
\right)
\nonumber\\
&=
\left(
1+
\exp\!\left[
-\frac{
\log \zeta_{i,j}
+
\log U_{i,j}
-
\log(1-U_{i,j})
}{\tau_g}
\right]
\right)^{-1}
\nonumber\\
&=
\frac{
\zeta_{i,j}^{1/\tau_g}U_{i,j}^{1/\tau_g}
}{
\zeta_{i,j}^{1/\tau_g}U_{i,j}^{1/\tau_g}
+
(1-U_{i,j})^{1/\tau_g}
},
\qquad
\tau_g>0 .
\label{eq:gumbel_sigmoid_full_clean}
\end{align}

Thus masking is exact at both the hard and relaxed levels:
\[
m_{i,j}^{\mathrm{exp}}=0
\quad\Longrightarrow\quad
\zeta_{i,j}=0
\quad\Longrightarrow\quad
s_{i,j}^{\mathrm{hard}}=\tilde s_{i,j}=0
\qquad
\forall \tau_g>0 .
\]

We use the straight-through estimator
\[
s_{i,j}^{\mathrm{ST}}
=
s_{i,j}^{\mathrm{hard}}
-
\operatorname{sg}(\tilde s_{i,j})
+
\tilde s_{i,j},
\]
where \(\operatorname{sg}(\cdot)\) denotes stop-gradient. The forward pass uses
the hard selector \(s_{i,j}^{\mathrm{hard}}\), while the backward pass follows
the pathwise Binary Concrete gradient. For every feasible expert
\(j\in S_i\), conditioning on the sampled noise and writing
\[
a_{i,j}^{g}
=
\frac{\bar\gamma_{i,j}+\eta_{i,j}}{\tau_g},
\qquad
\tilde s_{i,j}=\sigma(a_{i,j}^{g}),
\]
the map \(\bar\gamma_{i,j}\mapsto \tilde s_{i,j}\) is \(C^\infty\), and
\[
\frac{\partial \tilde s_{i,j}}{\partial \bar\gamma_{i,j}}
=
\frac{1}{\tau_g}
\tilde s_{i,j}(1-\tilde s_{i,j}).
\]
Consequently,
\[
\nabla_{\theta}s_{i,j}^{\mathrm{ST}}
=
\nabla_{\theta}\tilde s_{i,j}
=
\frac{1}{\tau_g}
\tilde s_{i,j}(1-\tilde s_{i,j})
\nabla_{\theta}\gamma_{i,j},
\qquad j\in S_i,
\]
and the gradient is identically zero for infeasible experts.

\paragraph{Stage II: masked expert allocation.}
Conditional on the feasible set, allocation logits are temperature-scaled and
masked before normalization:
\[
\bar\beta_{i,j}
=
\frac{\beta_{i,j}}{\tau_a}
+
\log m_{i,j}^{\mathrm{exp}}
=
\begin{cases}
\beta_{i,j}/\tau_a, & j\in S_i,\\
-\infty, & j\notin S_i,
\end{cases}
\qquad
\tau_a>0.
\]
For \(k_i>0\), the masked softmax allocation is
\[
a_{i,j}
=
\frac{
m_{i,j}^{\mathrm{exp}}\exp(\beta_{i,j}/\tau_a)
}{
\sum_{k=1}^{M}
m_{i,k}^{\mathrm{exp}}\exp(\beta_{i,k}/\tau_a)
},
\qquad
a_i\in\Delta(S_i).
\]
Thus allocation mass is normalized only over experts that are actually feasible
for case \(i\).

\paragraph{Full policy.}
Since independent Bernoulli support selection can produce an empty selected set,
we repair the realized forward support by falling back to the full feasible mask:
\[
s_i^{\sharp}
=
\begin{cases}
s_i^{\mathrm{hard}}, & \mathbf 1^\top s_i^{\mathrm{hard}}>0,\\
m_i^{\mathrm{exp}}, & \mathbf 1^\top s_i^{\mathrm{hard}}=0.
\end{cases}
\]
For \(k_i>0\), the conditional expert policy is the renormalized restriction of
the masked allocation \(a_i\) to the repaired support:
\begin{equation}
q_{i,j}
=
\frac{
a_{i,j}s_{i,j}^{\sharp}
}{
\sum_{k=1}^{M}a_{i,k}s_{i,k}^{\sharp}
},
\qquad
j=1,\ldots,M .
\label{eq:conditional_expert_allocation}
\end{equation}
Combining the soft deferral mass with the conditional expert distribution gives
the full action policy
\begin{equation}
\label{eq:full_policy}
\pi_i
=
\bigl(1-d_i,\;d_iq_i\bigr)
\in\Delta(\mathcal A_i),
\qquad
q_i\in\Delta(S_i),
\qquad
\sum_{j=0}^{M}\pi_{i,j}=1.
\end{equation}%
\vspace{-4mm}
\subsection{Group-specific distribution prior}
\label{subsec:gsdp}
GSDP regularizes deferred expert allocation at the group level. Each sample is assigned a hybrid group
\[
g(i)=(f(i),c(i))\in\{1,\ldots,G\},
\]
where \(f(i)\) is an availability-support family induced by \(m_i^{\mathrm{exp}}\), and \(c(i)\) is a within-family representation cluster. Define
\[
\vspace{-2mm}
\mathcal I_g=\{i:g(i)=g\},\qquad D_g=\sum_{i\in\mathcal I_g}d_i,\qquad D_+=\sum_{i=1}^N d_i,\qquad 
\tilde q_{g,j}=D_g^{-1}\sum_{i\in\mathcal I_g}d_iq_{i,j}\quad(D_g>0).
\]
Thus \(\tilde q_g\) is the deferred-load-weighted allocation barycenter for group \(g\).

\paragraph{Hierarchical empirical-Bayes-style prior.} We construct \(\tilde p_g\) through a hierarchical empirical-Bayes-style shrinkage estimator that pools expert-reliability statistics from the global cohort to support families and then to fine-grained groups. For \(S\subseteq\mathcal T\), let
\[
\mathcal A_S=\{j:\exists i\in S\ \mathrm{s.t.}\ m_{i,j}^{\mathrm{exp}}=1\},\qquad
b_{S,j}=c_{\mathrm{fn}}\widehat{\mathrm{FNR}}_{S,j}^{\mathrm{Lap}}+c_{\mathrm{fp}}\widehat{\mathrm{FPR}}_{S,j}^{\mathrm{Lap}}+\kappa_j .
\]
For hierarchy level \(\ell\in\{\mathrm{glob},\mathrm{fam},\mathrm{grp}\}\), with eligible support \(\mathcal V_S^{(\ell)}\subseteq\mathcal A_S\), define
\[
\nu_{S,j}^{(\ell)}=\frac{\mathbf 1[j\in\mathcal V_S^{(\ell)}]v_j\exp(-\tau_{\mathrm{bad}}b_{S,j})}{\sum_{k\in\mathcal V_S^{(\ell)}}v_k\exp(-\tau_{\mathrm{bad}}b_{S,k})},\qquad
\hat\nu_S^{(\ell)}=(1-u_\ell)\nu_S^{(\ell)}+u_\ell\,\mathrm{Unif}(\mathcal V_S^{(\ell)}).
\]
Here \(v_j>0\) is an optional capacity multiplier, \(\tau_{\mathrm{bad}}>0\) controls concentration toward low-badness experts, and \(u_\ell\in[0,1]\) gives a uniform floor.

Let \(\mathcal T_f=\{i:F_i=f\}\), \(\mathcal T_g=\{i:G_i=g\}\), \(n_f=|\mathcal T_f|\), \(n_g=|\mathcal T_g|\), and \(f(g)\) be the parent family. With \(\propto_{\mathcal A}\) denoting restriction to \(\mathcal A\) and renormalization, define
\[
\small
p^{\mathrm{glob}}=\hat\nu_{\mathcal T}^{(\mathrm{glob})}\;,\;
p_f\propto_{\mathcal A_{\mathcal T_f}}a_f\hat\nu_{\mathcal T_f}^{(\mathrm{fam})}+(1-a_f)p^{\mathrm{glob}}\;,\;
p_g^{(0)}\propto_{\mathcal A_{\mathcal T_g}}a_g\hat\nu_{\mathcal T_g}^{(\mathrm{grp})}+a_g^{\mathrm{fam}}p_{f(g)}+\rho_{\mathrm{glob}}p^{\mathrm{glob}} .
\]
The shrinkage weights are
\[
a_f=\frac{n_f}{n_f+n_{\mathrm{fam}}^{(0)}},\qquad
a_g=\frac{n_g}{n_g+n_{\mathrm{grp}}^{(0)}},\qquad
a_g^{\mathrm{fam}}=[1-a_g-\rho_{\mathrm{glob}}]_+ .
\]
Here \(n_{\mathrm{fam}}^{(0)}\) and \(n_{\mathrm{grp}}^{(0)}\) are pseudo-counts, while \(\rho_{\mathrm{glob}}\) preserves global bleed-through for small or noisy groups. Finally,
\[
\tilde p_g=p_g^{(0)}\in\Delta(\mathcal A_{\mathcal T_g}),\qquad \tilde p_{g,j}=0\ \ \forall j\notin\mathcal A_{\mathcal T_g}.
\]
\paragraph{GSDP regularizer.} For \(\mathcal G_+=\{g:D_g>0\}\) and \(\omega_g=D_g/D_+\), define, when \(D_+>0\),
\[
\mathcal L_{\mathrm{GSDP}}(\theta)=\sum_{g\in\mathcal G_+}\omega_gD_{\mathrm{KL}}(\tilde q_g\|\tilde p_g)
=\sum_{g\in\mathcal G_+}\frac{D_g}{D_+}\sum_{j\in\mathcal A_{\mathcal T_g}}\tilde q_{g,j}\log\frac{\tilde q_{g,j}}{\tilde p_{g,j}},
\]
and set \(\mathcal L_{\mathrm{GSDP}}=0\) if \(D_+=0\). Thus GSDP acts only on deferred mass, weights groups by effective deferred load, and shrinks each group allocation toward a reliability-weighted hierarchical prior.
\subsection{Rank-majorization Jensen--Shannon regularizer}
\label{subsec:rank_js}
GSDP controls group-level allocation, but does not directly prevent sample-level concentration. We therefore regularize the sorted rank profile of \(q_i\), independent of expert identity. For \(k_i=|S_i|>0\), define
\vspace{-2mm}
\begin{equation}
r_i = \operatorname{sort}_{\downarrow}(q_i|_{S_i}) = (r_{i,1},\ldots,r_{i,k_i}), \qquad
r_{i,1} \ge \cdots \ge r_{i,k_i} \ge 0, \qquad
\sum_{t=1}^{k_i} r_{i,t} = 1 .
\end{equation}
For $\varrho \in (0,1)$, the truncated geometric reference is
\vspace{-2mm}
\[
g_{k_i,\varrho}(t)=\frac{(1-\varrho)\varrho^{t-1}}{1-\varrho^{k_i}},\qquad t=1,\ldots,k_i .
\]
\vspace{-2mm}
Let
\[
R_i(t)=\sum_{s=1}^t r_{i,s},\qquad G_i(t)=\sum_{s=1}^t g_{k_i,\varrho}(s),\qquad
\chi_i=\mathbf 1\!\left[\max_{1\le t\le k_i}(R_i(t)-G_i(t))>m_{\mathrm{maj}}\right].
\]
Thus the penalty activates only when \(q_i\) places excessive cumulative mass on the largest ranks. With
\[
b_i^\star=\frac12(r_i+g_{k_i,\varrho}),\qquad b_{i,t}^\star=\frac12(r_{i,t}+g_{k_i,\varrho}(t)),
\]
the rank-profile Jensen--Shannon divergence is
\[
\operatorname{JS}(r_i\|g_{k_i,\varrho})=
\frac12\sum_{t=1}^{k_i}\left[
r_{i,t}\log\frac{r_{i,t}}{b_{i,t}^\star}
+g_{k_i,\varrho}(t)\log\frac{g_{k_i,\varrho}(t)}{b_{i,t}^\star}
\right].
\]
The deferred-mass-weighted rank regularizer is, for \(D_+>0\),
\[
\mathcal L_{\mathrm{rank}}(\theta)=\frac{1}{D_+}\sum_{i=1}^N d_i\chi_i\operatorname{JS}(r_i\|g_{k_i,\varrho}),
\]
and \(\mathcal L_{\mathrm{rank}}=0\) if \(D_+=0\). Hence the penalty acts on the conditional expert allocation only in proportion to the sample's soft deferral mass.
\subsection{Soft deferral control and full objective}
\label{subsec:full_objective}
\vspace{-2mm}
Let \(\bar d=N^{-1}\sum_{i=1}^N d_i\). For fixed \(q_i\), increasing the
soft deferral mass \(d_i\) is cost-improving exactly when the expert-side
cost advantage over the AI action is negative.To prevent cost minimization from over-concentrating deferrals under sparse labels and heterogeneous masks,
we optimize a GSDP/Rank-JS-regularized objective together with a
one-sided augmented-Lagrangian deferral controller for the soft
operational budget \(\bar d\le\rho_{\mathrm{def}}\):

\begin{equation}
\begin{aligned}
\min_{\theta}\ \mathcal J(\theta,\lambda_{\mathrm{def}})
&=
\frac1N
\sum_{i=1}^{N}
\left[
C_i^{\mathrm{ai}}
+
d_i
\left(
\sum_{j=1}^{M}
q_{i,j}
\left(
C_{i,j}^{\mathrm{exp}}
+
\gamma\kappa_j
\right)
-
C_i^{\mathrm{ai}}
\right)
\right] \\
&\quad
+
\lambda_{\mathrm{GSDP}}\mathcal L_{\mathrm{GSDP}}(\theta)
+
\lambda_{\mathrm{rank}}\mathcal L_{\mathrm{rank}}(\theta)
+
\lambda_{\mathrm{def}}\left(\bar d-\rho_{\mathrm{def}}\right)
+
\frac{\mu}{2}
\left[\bar d-\rho_{\mathrm{def}}\right]_+^2 ,
\end{aligned}
\label{eq:full_augmented_objective}
\end{equation}
where \(\lambda_{\mathrm{GSDP}},\lambda_{\mathrm{rank}}\ge0\),
\(\lambda_{\mathrm{def}}\ge0\) is the augmented-Lagrangian multiplier, and
\(\mu>0\) controls the quadratic penalty for budget violation. After each
epoch, the deferral multiplier is updated by projected ascent,
\[
\lambda_{\mathrm{def}}
\leftarrow
\left[
\lambda_{\mathrm{def}}
+
\eta_{\lambda}
\left(\bar d-\rho_{\mathrm{def}}\right)
\right]_+ .
\]
\vspace{-10mm}
\section{Experiments}
\label{sec:experiments}
\vspace{-2mm}
\textbf{Datasets and benchmark construction.}
We evaluate on three glaucoma cohorts: REFUGE~\citep{fang2022refuge2},
CHAKSU~\citep{kumar2023chakṣu}, and ORIGA~\citep{zhang2010origa}. The base AI
classifier is a SwinV2 model~\citep{liu2022swin} trained only on the REFUGE
training split and kept frozen during router training and evaluation, allowing us to study deferral under cross-cohort deployment shift. The cohorts provide heterogeneous supervision: CHAKSU includes native decisions from five
ophthalmologists, REFUGE contains expert optic-disc and optic-cup annotations,
and ORIGA has no native expert decisions. We treat the CHAKSU and REFUGE graders as disjoint
experts, yielding a unified expert set of size \(M=12\). Each sample is assigned an availability mask \(m_i^{\mathrm{exp}}\), where
\(m_{i,j}^{\mathrm{exp}}=1\) iff expert \(j\) has a decision for sample \(i\). Pseudo-labeling and dataset-specific supervision details are provided in Appendix~\ref{app:pseudo_labeling}.
\vspace{-2mm}
\subsection{Experiment Results}
\vspace{-2mm}
\textbf{Overall comparison.} Table~\ref{tab:overall} shows that MPD$^2$-Router achieves the most favorable system-level trade-off among the compared selective-prediction and learning-to-defer methods. Relative to AI-only inference, it improves accuracy from \(0.755\) to \(0.960\), F1 from \(0.530\) to \(0.868\), and MCC from \(0.482\) to \(0.844\), while reducing clinical cost from \(0.372\) to \(0.068\). This improvement is clinically meaningful because the router reduces high-cost autonomous AI errors rather than merely increasing average predictive performance. Importantly, the gain is not driven by indiscriminate review: MPD$^2$-Router defers \(43.7\%\) of cases, compared with \(74.7\%\) for Verma-OvA and \(62.6\%\) for Hemmer-MoE, and its spatial deferral pattern concentrates on regions where the frozen AI is unreliable rather than spreading uniformly across the test distribution (Appendix Fig.~\ref{fig:spatial_performance_map}). In the F1--cost and MCC--cost planes, MPD$^2$-Router is Pareto-favorable: no competing method simultaneously attains higher predictive utility and lower total cost (Appendix~\ref{app:pareto_tradeoffs}).
\begin{table*}[t]
\caption{
Overall and dataset-wise test performance. 
All methods are evaluated on the same overall test set ($N=899$); 
dataset sizes are REFUGE ($N=400$), CHAKSU ($N=336$), and ORIGA ($N=163$).
MPD$^2$-Router achieves the strongest overall performance and remains robust across shifted cohorts.
}
\label{tab:overall_and_dataset_breakdown}
\centering
\scriptsize
\setlength{\tabcolsep}{3.0pt}
\begin{subtable}{\textwidth}
\caption{Overall test performance.}
\label{tab:overall}
\centering
\vspace{-0.5em}
\resizebox{\textwidth}{!}{%
\begin{tabular}{lrrrrrrrrrr}
\toprule
Method & Acc & Prec & F1 & Sens & Spec & MCC & Defer & ClinicalCost & ExpertCost & TotalCost \\
\midrule
AI-Only 
& 0.755 & 0.370 & 0.530 & \textbf{0.932} & 0.725 & 0.482 & 0.000 & 0.372 & 0.000 & 0.372 \\
Mao-TwoStage 
& 0.925 & 0.726 & 0.760 & 0.797 & 0.948 & 0.717 & 0.403 & 0.127 & 0.084 & 0.211 \\
Narasimhan-PH 
& 0.937 & 0.764 & 0.794 & 0.827 & 0.956 & 0.758 & 0.388 & 0.108 & 0.079 & 0.187 \\
Verma-OvA
& 0.949 & 0.813 & 0.831 & 0.850 & 0.966 & 0.801 & 0.747 & 0.088 & 0.200 & 0.288 \\
Hemmer-MoE 
& 0.949 & 0.804 & 0.833 & 0.865 & 0.963 & 0.804 & 0.626 & 0.087 & 0.185 & 0.271 \\
Keswani-Joint Committee 
& 0.932 & 0.786 & 0.764 & 0.744 & 0.965 & 0.725 & 0.468 & 0.121 & 0.143 & 0.264 \\
\textbf{MPD$^2$-Router} 
& \textbf{0.960} & \textbf{0.849} & \textbf{0.868} & 0.887 & \textbf{0.973} & \textbf{0.844} & 0.437 & \textbf{0.068} & 0.113 & \textbf{0.181} \\
\bottomrule
\end{tabular}}
\end{subtable}

\vspace{0.5em}

\begin{subtable}{\textwidth}
\caption{Dataset-wise test performance.}
\label{tab:per_dataset}
\centering
\vspace{-0.2em}
\resizebox{\textwidth}{!}{%
\begin{tabular}{llrrrrrrrrrr}
\toprule
Method & Dataset & Acc & Prec & F1 & Sens & Spec & MCC & Defer & ClinicalCost & ExpertCost & TotalCost \\
\midrule
\multirow{3}{*}{AI-Only}
& REFUGE & 0.882 & 0.458 & 0.618 & 0.950 & 0.875 & 0.610 & 0.000 & 0.179 & 0.000 & 0.179 \\
& CHAKSU & 0.732 & 0.353 & 0.511 & 0.922 & 0.698 & 0.455 & 0.000 & 0.408 & 0.000 & 0.408 \\
& ORIGA  & 0.491 & 0.328 & 0.484 & 0.929 & 0.339 & 0.263 & 0.000 & 0.773 & 0.000 & 0.773 \\

\midrule
\multirow{3}{*}{Mao-TwoStage}
& REFUGE & 0.938 & 0.642 & 0.731 & 0.850 & 0.947 & 0.705 & 0.220 & 0.101 & 0.053 & 0.154 \\
& CHAKSU & 0.940 & 0.816 & 0.800 & 0.784 & 0.968 & 0.765 & 0.470 & 0.106 & 0.081 & 0.186 \\
& ORIGA  & 0.865 & 0.727 & 0.744 & 0.762 & 0.901 & 0.653 & 0.712 & 0.233 & 0.168 & 0.401 \\

\midrule
\multirow{3}{*}{Narasimhan-PH}
& REFUGE & 0.950 & 0.692 & 0.783 & 0.900 & 0.956 & 0.763 & 0.180 & 0.080 & 0.041 & 0.121 \\
& CHAKSU & 0.949 & 0.886 & 0.821 & 0.765 & 0.982 & 0.795 & 0.470 & 0.094 & 0.081 & \textbf{0.175} \\
& ORIGA  & 0.877 & 0.729 & 0.778 & 0.833 & 0.893 & 0.697 & 0.730 & 0.206 & 0.168 & 0.374 \\

\midrule
\multirow{3}{*}{Verma-OvA}
& REFUGE & 0.960 & 0.773 & 0.810 & 0.850 & 0.972 & 0.788 & 0.598 & 0.068 & 0.175 & 0.243 \\
& CHAKSU & 0.940 & 0.782 & 0.811 & 0.843 & 0.958 & 0.777 & 0.866 & 0.101 & 0.206 & 0.307 \\
& ORIGA  & 0.939 & 0.900 & 0.878 & 0.857 & 0.967 & 0.838 & 0.871 & 0.110 & 0.249 & 0.360 \\

\midrule
\multirow{3}{*}{Hemmer-MoE}
& REFUGE & 0.953 & 0.756 & 0.765 & 0.775 & 0.972 & 0.739 & 0.465 & 0.083 & 0.133 & 0.215 \\
& CHAKSU & 0.943 & 0.776 & 0.826 & 0.882 & 0.954 & 0.794 & 0.688 & 0.094 & 0.213 & 0.306 \\
& ORIGA  & 0.951 & 0.886 & 0.907 & 0.929 & 0.959 & 0.874 & 0.896 & 0.083 & 0.254 & 0.337 \\

\midrule
\multirow{3}{*}{Keswani-Joint Committee}
& REFUGE & 0.945 & 0.765 & 0.703 & 0.650 & 0.978 & 0.675 & 0.110 & 0.100 & 0.032 & 0.132 \\
& CHAKSU & 0.938 & 0.742 & 0.814 & 0.902 & 0.944 & 0.782 & 0.866 & 0.101 & 0.268 & 0.370 \\
& ORIGA  & 0.890 & 0.900 & 0.750 & 0.643 & 0.975 & 0.698 & 0.528 & 0.212 & 0.157 & 0.369 \\

\midrule
\multirow{3}{*}{\textbf{MPD$^2$-Router}}
& REFUGE & 0.975 & 0.857 & 0.878 & 0.900 & 0.983 & 0.864 & 0.245 & 0.043 & 0.070 & \textbf{0.113} \\
& CHAKSU & 0.958 & 0.878 & 0.860 & 0.843 & 0.979 & 0.836 & 0.491 & 0.074 & 0.114 & 0.188 \\
& ORIGA  & 0.926 & 0.812 & 0.867 & 0.929 & 0.926 & 0.820 & 0.798 & 0.120 & 0.215 & \textbf{0.335} \\
\bottomrule
\end{tabular}}
\end{subtable}
\end{table*}

\textbf{Dataset-wise robustness.}
Table~\ref{tab:per_dataset} further disaggregates performance across REFUGE, CHAKSU, and ORIGA. MPD$^2$-Router attains the lowest total cost on REFUGE and ORIGA, and remains close to the best method on CHAKSU, while maintaining strong F1 and MCC across all three cohorts. The ORIGA results are particularly informative because the frozen AI degrades sharply under distribution shift, with accuracy falling to \(0.491\) and clinical cost rising to \(0.773\). MPD$^2$-Router recovers ORIGA performance to \(0.926\) accuracy, \(0.867\) F1, and \(0.820\) MCC, while reducing clinical cost to \(0.120\) and achieving the lowest total cost. Competing methods can obtain high per-domain accuracy or F1, but often at the price of substantially higher deferral and expert cost; for example, Hemmer-MoE defers \(89.6\%\) of ORIGA cases. Thus, MPD$^2$-Router provides a more deployable operating point: it is not uniformly best on every isolated metric, but it yields consistently strong, cost-sensitive, and shift-robust routing decisions.

\textbf{Risk-stratified robustness.} Figure~\ref{fig:risk_stratified_cost} stratifies samples along two clinically meaningful axes and compares MPD\(^2\)-Router against baselines and the available human experts. Under model-reliability risk, defined by OOD and AI-uncertainty signals, the mean available-human cost remains relatively stable across quintiles, whereas AI-Only deteriorates sharply as reliability risk increases. Under structural glaucoma risk, clinical cost increases for nearly all methods, indicating that genuine clinical difficulty for both automated and human predictors. MPD\(^2\)-Router is comparatively resilient on both axes: its cost grows more slowly than most baselines and remains substantially below AI-Only in the highest-risk quintiles.
\begin{figure}[t]
    \centering
    \includegraphics[width=\linewidth]{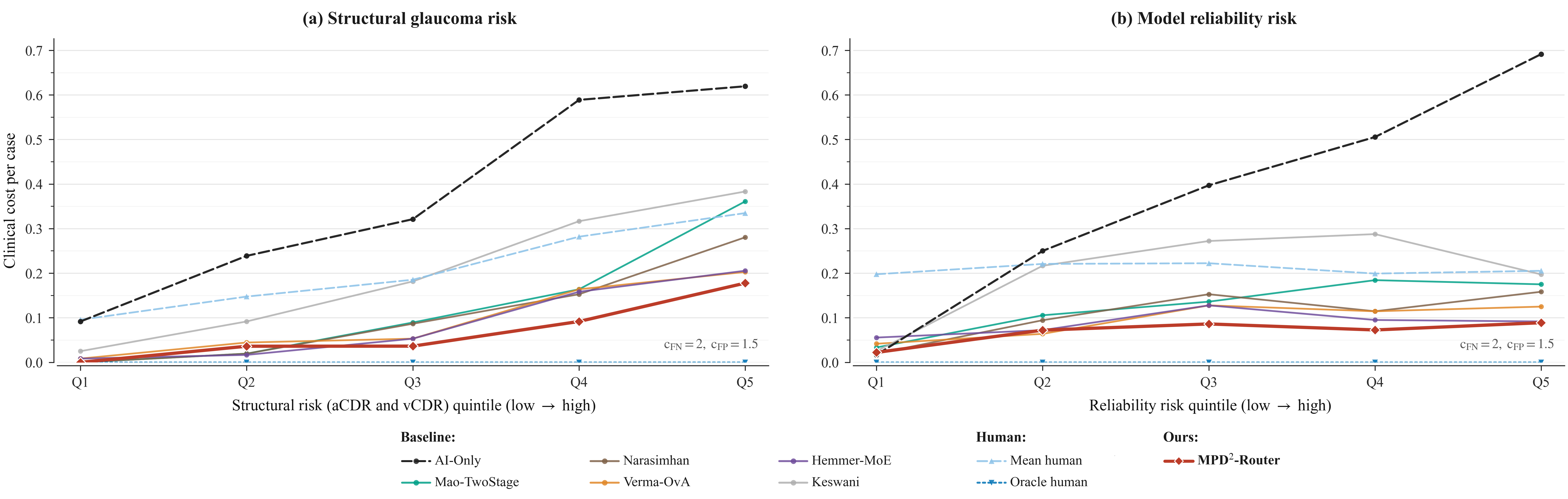}
    \caption{
    Risk-stratified clinical cost under structural glaucoma risk and
    model-reliability risk. Structural risk is defined by aCDR and vCDR
    quintiles, whereas reliability risk is defined by OOD and AI-uncertainty
    signals. MPD\(^2\)-Router exhibits slower cost growth than most baselines
    and remains markedly below AI-Only in the highest-risk strata.
    }
    \label{fig:risk_stratified_cost}
\end{figure}

\textbf{Expert collapse diagnostic.}
Table~\ref{tab:expert_collapse} examines whether human-routed cases concentrate on a small subset of experts. MPD$^2$-Router attains the lowest Top-1 share ($24.7\%$) and Top-2 shares($44.0\%$), indicating that its gains do not stem from collapse onto a dominant reader. While MPD$^2$-Router does not dominate Narasimhan-PH on every diversity metric, strict uniformity is neither clinically necessary nor desirable; the appropriate criterion is the avoidance of pathological concentration while preserving cost- and availability-aware allocation. By this measure, MPD$^2$-Router's diversity profile remains close to Narasimhan-PH and substantially improves over Hemmer-MoE and Keswani-Joint Committee, both of which exhibit severe routing concentration.
\begin{table}[t]
\centering
\caption{
Expert-collapse diagnostics over human-routed samples. Top-1 Share / Top-2 Share are the routed-case fractions assigned to the most-used and two most-used experts; lower Entropy-Collapse and Gini$_{\mathrm{norm}}$ and higher $N_{\mathrm{eff}}$ (number of effective experts) all indicate less concentrated routing.
}
\label{tab:expert_collapse}
\small
\resizebox{\linewidth}{!}{%
\begin{tabular}{lcccccc}
\toprule
Method & Routed $n$ & Top-1 Share $\downarrow$ & Top-2 Share $\downarrow$ & Entropy-Collapse $\downarrow$ & $N_{\mathrm{eff}}\uparrow$ & Gini$_{\mathrm{norm}}\downarrow$ \\
\midrule
Mao-TwoStage             & 362 & 0.340 & 0.594 & 0.300 & 4.352 & 0.692 \\
Narasimhan-PH            & 349 & 0.255 & 0.458 & \textbf{0.211} & \textbf{6.093} & \textbf{0.586} \\
Verma-OvA                & 672 & 0.339 & 0.679 & 0.378 & 3.750 & 0.759 \\
Hemmer-MoE               & 563 & 0.506 & 0.790 & 0.480 & 2.816 & 0.839 \\
Keswani-Joint Committee  & 421 & 0.876 & 0.976 & 0.822 & 1.284 & 0.973 \\
\textbf{MPD$^2$-Router}  & 393 & \textbf{0.247} & \textbf{0.440} & 0.225 & 6.027 & 0.595 \\
\bottomrule
\end{tabular}%
}
\end{table}
\vspace{-2mm}
\subsection{Ablation studies}
\vspace{-2mm}
We conducted three complementary ablation studies to isolate the contribution of the objective terms, the routing-regularization losses, and the augmented-Lagrangian (AL) deferral controller.

The objective ablation in Table~\ref{tab:ablation_objective} attains the most clinically meaningful operating point: it preserves strong predictive performance while keeping expert cost and deferral under control. Removing the tier-cost term yields marginally higher F1/MCC within one standard deviation, but this improvement is operationally expensive, increasing ExpertCost from 0.133 to 0.169 and raising soft deferral from 0.517 to 0.633. The clinical-only variant is more problematic: it nearly collapses into indiscriminate deferral (96.4\%), substantially higher ExpertCost 0.259, and worse Acc/F1/MCC. 

In contrast, ablating GSDP, Rank-JS, or both yields only negligible changes in predictive performance, indicating that these routing regularizers do not achieve routing diversity by sacrificing diagnostic accuracy. Instead they act primarily on expert-utilisation structure and improve reliability of the learned policy. Table~\ref{tab:ablation_regularization} supports this interpretation. Removing GSDP increases Top1$_{\mathrm{share}}$ from $0.371$ to $0.455$ and lowers the effective number of experts from $7.603$ to $6.754$, while removing Rank-JS produces a milder but similar concentration effect. Removing both regularizers is substantially worse: Top1$_{\mathrm{share}}$ rises to $0.493$, Top2$_{\mathrm{share}}$ to $0.714$, HHI$_{\mathrm{norm}}$ nearly doubles, and entropy drops from $0.815$ to $0.724$. Hence, GSDP and Rank-JS act complementarily: GSDP anchors group-level allocation to clinically informed priors, while Rank-JS discourages sample-level top-heavy routing, together preventing expert collapse without materially sacrificing clinical performance.

The augmented-Lagrangian in Table~\ref{tab:ablation_defer_control} demonstrates that deferral is controllable rather than an uncontrolled byproduct of the learned router. As the requested target increases from 0.25 to 0.70, the realized soft deferral rises monotonically from 0.239 to 0.517, with corresponding increases in hard deferral.The negative gaps indicate that the learned policy satisfies the request upper targest and remains conservative relative to the targets, which is clinically preferable and pragmatic to systematic over-deferral. Importantly, performance is only meaningfully degraded under very restrictive budgets such as AL25/AL30, where the router is forced to retain too many difficult cases; once the budget enters a clinically useful range, the model recovers strong performance, with AL40–AL70 giving Acc 0.942–0.947 and F1 0.813–0.825, close to the unconstrained NoAL result. This supports the claim that the AL term provides practical deferral control with limited performance compromise in the relevant operating regime.
\begin{table}[t]
\centering
\captionsetup[subtable]{skip=2pt}
\caption{Ablation studies on the test set.
(a) Objective ablation;
(b) routing-regularisation ablation;
(c) augmented-Lagrangian deferral-control.
Each entry reports mean $\pm$ standard deviation over 10 random seeds.
Lower Top1/Top2 Share, HHI, LoadCV, and DeadFrac, and higher EffExp/Entropy, indicate better expert utilisation.
In (c), Gap is corresponding defer minus the requested target.
}
\label{tab:ablation_all}
\begin{subtable}{\textwidth}
\centering
\subcaption{Objective ablation.}
\label{tab:ablation_objective}
\resizebox{\textwidth}{!}{%
\begin{tabular}{lccccccccc}
\toprule
Condition & $\mathcal{L}_{\mathrm{clin}}$ & ExpertCost & Defer$_{soft}$ & Defer$_{hard}$ & Sens. & Spec. & Acc & F1 & MCC \\
\midrule
Full objective & $0.091\pm0.009$ & $0.133\pm0.008$ & $0.517\pm0.031$ & $0.477\pm0.020$ & $0.851\pm0.023$ & $0.963\pm0.006$ & $0.947\pm0.006$ & $0.825\pm0.018$ & $0.794\pm0.021$ \\
w/o Tier loss & $0.089\pm0.009$ & $0.169\pm0.014$ & $0.633\pm0.052$ & $0.516\pm0.024$ & $0.862\pm0.022$ & $0.962\pm0.004$ & $0.947\pm0.005$ & $0.829\pm0.016$ & $0.799\pm0.019$ \\
Clinical only & $0.108\pm0.012$ & $0.259\pm0.019$ & $0.959\pm0.053$ & $0.964\pm0.109$ & $0.867\pm0.018$ & $0.946\pm0.011$ & $0.934\pm0.008$ & $0.797\pm0.019$ & $0.762\pm0.022$ \\
w/o GSDP & $0.093\pm0.013$ & $0.139\pm0.009$ & $0.536\pm0.038$ & $0.485\pm0.016$ & $0.848\pm0.017$ & $0.962\pm0.008$ & $0.945\pm0.008$ & $0.822\pm0.024$ & $0.790\pm0.028$ \\
w/o Rank-JS & $0.094\pm0.009$ & $0.134\pm0.011$ & $0.520\pm0.040$ & $0.478\pm0.019$ & $0.853\pm0.021$ & $0.961\pm0.006$ & $0.945\pm0.005$ & $0.820\pm0.016$ & $0.789\pm0.019$ \\
w/o GSDP + Rank-JS & $0.091\pm0.009$ & $0.139\pm0.008$ & $0.523\pm0.026$ & $0.486\pm0.016$ & $0.856\pm0.020$ & $0.962\pm0.006$ & $0.946\pm0.005$ & $0.825\pm0.016$ & $0.795\pm0.019$ \\
\bottomrule
\end{tabular}}
\end{subtable}
\vspace{0.0em}
\begin{subtable}{\textwidth}
\centering
\subcaption{Routing-regularisation ablation.}
\label{tab:ablation_regularization}
\resizebox{\textwidth}{!}{%
\begin{tabular}{lcccccccccc}
\toprule
Condition & $\mathcal{L}_{\mathrm{clin}}$ & Defer$_{soft}$ & Defer$_{hard}$ & Top1$_{share}$ & Top2$_{share}$ & HHI$_{norm}$ & LoadCV$_{soft}$ & DeadFrac$_{soft}$ & EffExp & Entropy \\
\midrule
Full regularisation & $0.091\pm0.009$ & $0.517\pm0.031$ & $0.477\pm0.020$ & $0.371\pm0.062$ & $0.591\pm0.083$ & $0.086\pm0.016$ & $0.971\pm0.091$ & $0.208\pm0.067$ & $7.603\pm0.521$ & $0.815\pm0.028$ \\
w/o GSDP & $0.093\pm0.013$ & $0.536\pm0.038$ & $0.485\pm0.016$ & $0.455\pm0.067$ & $0.680\pm0.085$ & $0.123\pm0.024$ & $1.159\pm0.111$ & $0.233\pm0.050$ & $6.754\pm0.497$ & $0.768\pm0.030$ \\
w/o Rank-JS & $0.094\pm0.009$ & $0.520\pm0.040$ & $0.478\pm0.019$ & $0.398\pm0.057$ & $0.619\pm0.082$ & $0.094\pm0.018$ & $1.011\pm0.094$ & $0.250\pm0.065$ & $7.371\pm0.500$ & $0.803\pm0.028$ \\
w/o GSDP + Rank-JS & $0.091\pm0.009$ & $0.523\pm0.026$ & $0.486\pm0.016$ & $0.493\pm0.068$ & $0.714\pm0.074$ & $0.164\pm0.044$ & $1.334\pm0.174$ & $0.242\pm0.069$ & $6.082\pm0.680$ & $0.724\pm0.046$ \\
\bottomrule
\end{tabular}}
\end{subtable}
\vspace{0.0em}
\begin{subtable}{\textwidth}
\centering
\subcaption{Augmented-Lagrangian deferral-control.}
\label{tab:ablation_defer_control}
\resizebox{\textwidth}{!}{%
\begin{tabular}{lccccccccc}
\toprule
Condition & Target & Defer$_{soft}$ & Defer$_{hard}$ & Gap$_{soft}$ & Gap$_{hard}$ & $\mathcal{L}_{\mathrm{total}}$ & $\mathcal{L}_{\mathrm{clin}}$ & Acc & F1 \\
\midrule
AL25 & $0.250$ & $0.239\pm0.007$ & $0.201\pm0.006$ & $-0.011\pm0.007$ & $-0.049\pm0.006$ & $0.176\pm0.007$ & $0.162\pm0.007$ & $0.897\pm0.005$ & $0.721\pm0.010$ \\
AL30 & $0.300$ & $0.281\pm0.012$ & $0.236\pm0.014$ & $-0.019\pm0.012$ & $-0.064\pm0.014$ & $0.156\pm0.011$ & $0.139\pm0.012$ & $0.913\pm0.008$ & $0.751\pm0.017$ \\
AL40 & $0.400$ & $0.371\pm0.006$ & $0.342\pm0.009$ & $-0.029\pm0.006$ & $-0.058\pm0.009$ & $0.124\pm0.009$ & $0.098\pm0.010$ & $0.942\pm0.006$ & $0.813\pm0.018$ \\
AL50 & $0.500$ & $0.439\pm0.010$ & $0.417\pm0.015$ & $-0.061\pm0.010$ & $-0.083\pm0.015$ & $0.128\pm0.010$ & $0.097\pm0.010$ & $0.943\pm0.006$ & $0.814\pm0.018$ \\
AL60 & $0.600$ & $0.491\pm0.023$ & $0.459\pm0.022$ & $-0.109\pm0.023$ & $-0.141\pm0.022$ & $0.125\pm0.011$ & $0.091\pm0.011$ & $0.947\pm0.007$ & $0.825\pm0.020$ \\
AL70 & $0.700$ & $0.517\pm0.031$ & $0.477\pm0.020$ & $-0.183\pm0.031$ & $-0.223\pm0.020$ & $0.126\pm0.010$ & $0.091\pm0.009$ & $0.947\pm0.006$ & $0.825\pm0.018$ \\
NoAL & -- & $0.535\pm0.034$ & $0.503\pm0.008$ & -- & -- & $0.127\pm0.009$ & $0.092\pm0.009$ & $0.946\pm0.006$ & $0.824\pm0.016$ \\
\bottomrule
\end{tabular}}
\end{subtable}
\end{table}
\vspace{-3mm}
\section{Conclusion}
\label{sec:conclusion}
\vspace{-2mm}
We presented MPD$^2$-Router for cost-sensitive glaucoma triage. By separating deferral from expert allocation, enforcing per-sample human availability, and regularizing expert use, MPD$^2$-Router addresses key deployment failures of standard L2D methods: indiscriminate deferral, expert collapse, asymmetric clinical harm, and distribution shift. Across three glaucoma cohorts, the proposed method improves diagnostic utility, lowers clinical and total cost, maintains controlled deferral, and yields more balanced expert utilization than strong baselines. These results support structured human--AI routing as a practical direction for safer ophthalmic AI deployment. A limitation is the scarcity of dense prospective multi-expert annotations, so our evaluation relies on retrospective cross-cohort datasets with heterogeneous and partially observed expert labels; prospective clinical validation and appropriate regulatory review remain necessary before deployment. However, this limitation reflects the real clinical setting our mask-aware formulation is designed for, and the cross-cohort results indicate that MPD$^2$-Router remains effective under heterogeneous and sparse expert supervision.
\clearpage
\bibliographystyle{plainnat}
\bibliography{citations}
\clearpage
\appendix
\section{Related Work}
\label{sec:related_work}
\paragraph{Selective prediction, algorithmic triage, and learning to defer.}
Learning to defer (L2D) is closely related to selective prediction and
classification with a reject option, where a model abstains on uncertain inputs
to improve the risk--coverage trade-off \citep{geifman2017selective, geifman2019selectivenet}. However, classical selective prediction treats
abstention as an unstructured reject action, whereas clinical deployment
requires deciding not only whether to abstain but also whether the downstream
human decision-maker is likely to improve the final outcome. This distinction
motivated early human--AI triage and deferral frameworks. \citet{madras2018predict} introduced learning to defer as a system-level alternative to pure automation,
showing that models can improve both accuracy and fairness by passing selected
cases to external decision-makers. \citet{raghu2019algorithmic} further argued
that automation is not simply a question of whether an algorithm outperforms
humans on average, but an instance-wise allocation problem that depends on both
algorithmic and human error. Subsequent work on differentiable triage formalized
this division of labor and showed that models trained for full automation may be
suboptimal when only a subset of cases will be handled by the algorithm
\citep{okati2021differentiable}.

\paragraph{Single-expert learning to defer.}
A central theoretical foundation for L2D was established by
\citet{mozannar2020consistent}, who introduced a Bayes-consistent surrogate for
single-expert deferral by casting the classify-or-defer decision as an augmented
cost-sensitive classification problem. \citet{verma2022calibrated} identified
calibration limitations in the softmax-style formulation and proposed a
one-vs-all (OvA) parameterization that yields calibrated estimates of expert
correctness while preserving consistency. \citet{cao2023defense} subsequently showed that the calibration failure is not inherent to softmax itself, but rather to symmetric surrogate design, and proposed an asymmetric softmax loss that can
recover both consistency and calibrated probability estimation. In parallel,
\citet{narasimhan2022post} developed post-hoc plug-in estimators that decouple base
classifier training from the deferral function, which addresses underfitting caused
by jointly optimizing classification and deferral under nontrivial deferral
costs.

\paragraph{Multi-expert deferral.}
Extensions from a single expert to multiple experts have proceeded along
several complementary axes. \citet{keswani2021towards} framed multi-expert
deferral as a joint classifier--allocator problem in which the system chooses
among the model and multiple human experts, with an emphasis on accuracy and
fairness under heterogeneous expert behavior. \citet{hemmer2022forming}
proposed a human--AI team formation approach that jointly learns classifier
predictions and expert allocation, showing that team performance can improve
when the model complements rather than duplicates human expertise.
\citet{verma2022calibrated} studied the statistical properties of multi-expert
L2D, deriving consistent softmax and OvA surrogates, analyzing confidence
calibration across experts, and introducing conformal expert-set selection.
\citet{mao2023two} established a two-stage multi-expert L2D framework with
\(\mathcal H\)-consistency and Bayes-consistency guarantees, where a predictor
is first trained and a deferral function is then learned to assign each input to
the most suitable expert. More recent theoretically grounded work has further
studied surrogate design, realizable consistency, and cost-sensitive deferral in
multi-expert settings \citep{mao2024principled,alves2024cost}. Lastly,
\citet{tailor2024learning} considered learning to defer to a population of
experts, using meta-learning to adapt to experts whose predictions were not
observed during training.

\paragraph{Clinical human--AI deferral.}
Clinical AI provides a particularly strong motivation for L2D because both AI
models and human experts are imperfect, and their errors may be complementary.
In medical imaging, Complementarity-Driven Deferral to Clinical Workflow
(CoDoC) showed that selective deferral between an AI model and a clinical
workflow can improve diagnostic accuracy and reduce workload in screening
settings \citep{dvijotham2023enhancing}. Ophthalmology is especially aligned
with this perspective: diagnostic decisions are affected by image quality,
disease severity, out-of-distribution acquisition conditions, and inter-grader
variability in optic-disc assessment \citep{varma1992expert,
pourjavan2024evaluating}. These properties make glaucoma screening a
natural setting for selective, cost-aware, and expert-aware routing rather than
uniform automation or indiscriminate referral.

\paragraph{Positioning of MPD\texorpdfstring{\(^2\)}{2}-Router.}
Despite substantial progress, existing L2D methods primarily emphasize
consistency, calibration, or post-hoc deferral, and most multi-expert
formulations assume a fixed expert set, dense expert annotations, or an
unregularized allocation over candidate experts. They provide limited machinery
for handling heterogeneous expert availability masks, incorporating
clinically-grounded priors over expert competence, controlling operational
deferral rates, or preventing routing mass from collapsing onto a small subset
of experts. MPD\(^2\)-Router addresses these gaps through a mask-aware,
prior-regularized dual-head architecture. Unlike single-softmax formulations,
it explicitly separates the decision of \emph{whether to defer} from the
conditional decision of \emph{whom to defer to}. The defer head estimates a
soft human-referral probability, while the expert head performs stochastic
support selection and masked allocation only over feasible experts. In addition,
our method incorporates fused clinical reliability signals, including AI logits,
uncertainty, image quality, OOD risk, and structural glaucoma biomarkers; uses
group-specific distribution priors and rank-based JS regularization to reduce
expert collapse; and enforces a soft deferral budget through an augmented
Lagrangian objective. Thus, MPD\(^2\)-Router \emph{shifts multi-expert L2D from a
purely accuracy-driven allocation problem toward a clinically aligned,
availability-aware, cost-aware, and collapse-resistant routing framework} for robust
human--AI glaucoma screening.
\section{Workflow and Dual-Head Architecture of MPD\texorpdfstring{\(^2\)}{2}-Router}
\label{app:workflow}
\begin{figure}[t]
    \centering
    \includegraphics[width=1.0\linewidth]{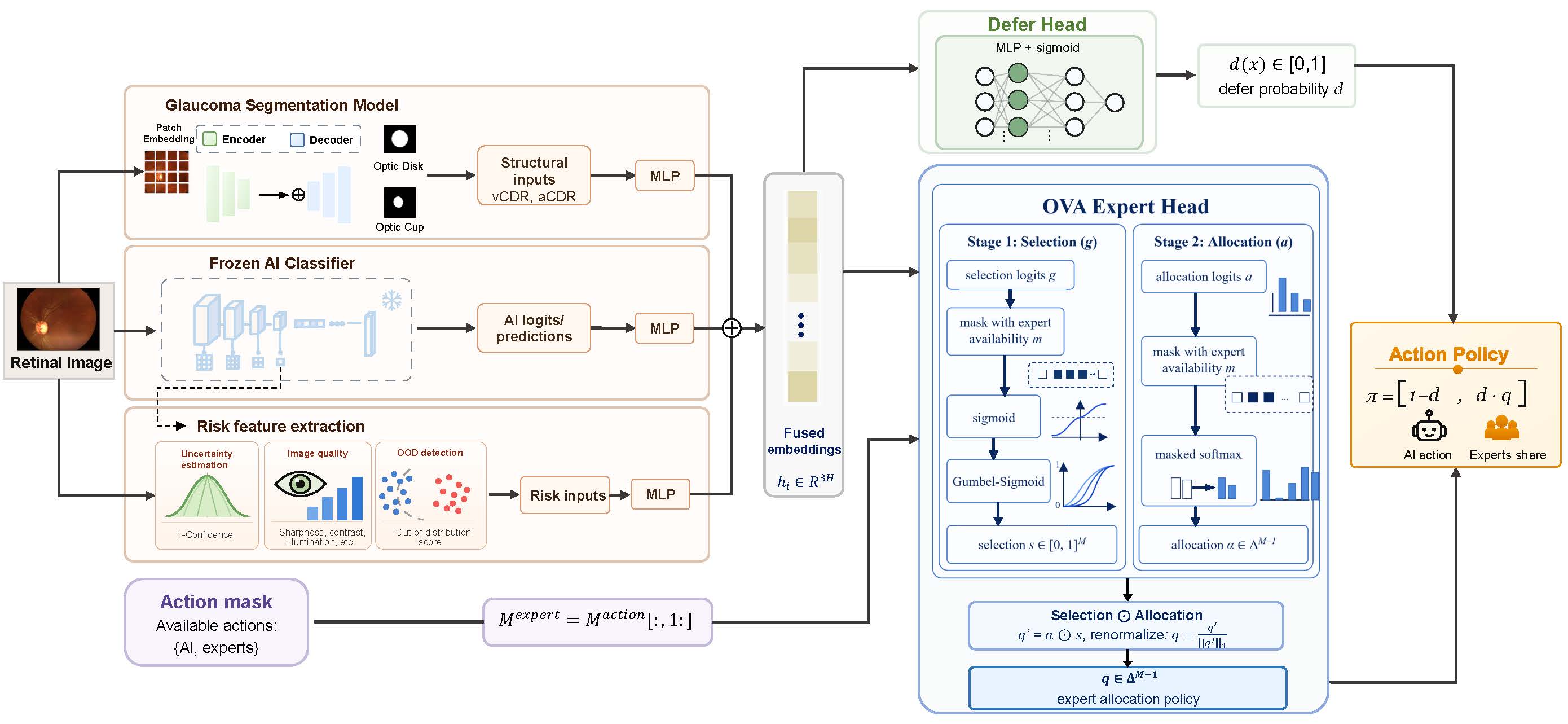}
    \caption{
    Overview of MPD\(^2\)-Router. A retinal image is processed by three
    complementary branches: a frozen AI classifier for diagnostic logits, a
    glaucoma segmentation model for structural biomarkers such as vCDR and
    aCDR, and risk-feature extractors for uncertainty, image quality, and OOD
    signals. These signals are fused into an embedding \(h_i\), which is passed
    to a dual-head router. The defer head estimates the probability \(d_i\) of
    human referral, while the mask-aware expert head first selects a feasible
    expert support and then allocates probability mass over available experts.
    The final policy \(\pi_\theta(x_i,m_i^{\mathrm{exp}})=(1-d_i,d_iq_i)\)
    therefore jointly determines whether the AI prediction is retained or the
    case is deferred to a specific human expert.
    }
    \label{fig:mpd2_workflow_app}
\end{figure}

Figure~\ref{fig:mpd2_workflow_app} illustrates the complete MPD\(^2\)-Router
workflow. Given a retinal image, the model extracts complementary diagnostic,
structural, and reliability signals from a frozen AI classifier, a glaucoma
segmentation model, and risk-feature modules. These signals are fused into a
shared representation \(h_i\), which feeds a dual-head routing architecture.
This dual-head design separates the decision of \emph{whether to defer} from
the decision of \emph{whom to defer to}: the defer head outputs a scalar
deferral probability \(d_i\), whereas the expert head produces a conditional,
mask-aware allocation distribution \(q_i\) over only the available experts.
Consequently, the final policy \(\pi_\theta=(1-d_i,d_iq_i)\) provides an
interpretable and feasibility-constrained mechanism for balancing AI retention
against expert-specific human referral.

\section{Deferred-load form of the GSDP penalty}
\label{app:gsdp_deferred_load}

We expand the group-specific distribution prior (GSDP) in its KL instantiation.
Let
\[
D_{+}:=\sum_{i=1}^{N}d_i
\]
denote the total soft deferred mass, and assume \(D_{+}>0\). For each group
\(g\in\{1,\ldots,G\}\), define the group-level deferred mass and deferred
expert load by
\[
D_g
=
\sum_{i:g(i)=g}d_i,
\qquad
Q_{g,j}
=
\sum_{i:g(i)=g}d_i q_{i,j},
\qquad
Q_g=(Q_{g,1},\ldots,Q_{g,M}).
\]
Since \(q_i\) is a normalized conditional distribution over feasible experts,
we have
\[
\sum_{j=1}^{M}Q_{g,j}
=
\sum_{i:g(i)=g}d_i
\sum_{j=1}^{M}q_{i,j}
=
D_g .
\]
For every active group \(g\) with \(D_g>0\), the normalized deferred expert-load
distribution is therefore
\[
\tilde q_{g,j}
=
\frac{Q_{g,j}}{D_g},
\qquad
\sum_{j=1}^{M}\tilde q_{g,j}=1.
\]
The GSDP penalty is
\[
\mathcal L_{\mathrm{GSDP}}(\theta)
=
\sum_{g=1}^{G}
\omega_g
D\!\left(
\tilde q_g \,\middle\|\, \tilde p_g
\right),
\qquad
\omega_g
=
\frac{D_g}{D_{+}}.
\]
Inactive groups with \(D_g=0\) have \(\omega_g=0\) and contribute no loss.

For the KL choice,
\[
D\!\left(
\tilde q_g \,\middle\|\, \tilde p_g
\right)
=
\sum_{j=1}^{M}
\tilde q_{g,j}
\log
\frac{\tilde q_{g,j}}{\tilde p_{g,j}},
\]
with the convention \(0\log 0=0\), and with the usual requirement that
\(\tilde p_{g,j}>0\) whenever \(\tilde q_{g,j}>0\). Substituting
\(\omega_g=D_g/D_{+}\) and \(\tilde q_{g,j}=Q_{g,j}/D_g\) gives
\begin{align}
\mathcal L_{\mathrm{GSDP}}^{\mathrm{KL}}(\theta)
&=
\sum_{g=1}^{G}
\frac{D_g}{D_{+}}
\sum_{j=1}^{M}
\frac{Q_{g,j}}{D_g}
\log
\frac{Q_{g,j}/D_g}{\tilde p_{g,j}}
\label{eq:app_gsdp_kl_expand_1}
\\
&=
\frac{1}{D_{+}}
\sum_{g=1}^{G}
\sum_{j=1}^{M}
Q_{g,j}
\log
\frac{Q_{g,j}}{D_g\tilde p_{g,j}} .
\label{eq:app_gsdp_kl_expand_2}
\end{align}
Thus the KL-GSDP penalty is equivalently a normalized divergence between the
actual group-wise deferred expert load \(Q_g\) and the target deferred load
\(D_g\tilde p_g\). In particular, for each group \(g\), the target is not merely
a probability vector \(\tilde p_g\), but a load vector whose total mass is scaled
by the amount of deferred mass assigned to that group.

\begin{corollary}[KL-GSDP as load matching]
\label{cor:gsdp_load_matching}
For a fixed active group \(g\) with \(D_g>0\), the KL-GSDP group term
\[
\Phi_g(Q_g)
=
\sum_{j=1}^{M}
Q_{g,j}
\log
\frac{Q_{g,j}}{D_g\tilde p_{g,j}}
\]
is minimized over the simplex-constrained load set
\[
\mathcal Q_g
=
\left\{
Q_g\in\mathbb R_{+}^{M}:
\sum_{j=1}^{M}Q_{g,j}=D_g
\right\}
\]
if and only if
\[
Q_g=D_g\tilde p_g .
\]
\end{corollary}

\begin{proof}
Since \(D_g>0\), write \(r_g=Q_g/D_g\). Then \(r_g\in\Delta^{M-1}\), and
\[
\Phi_g(Q_g)
=
D_g
\sum_{j=1}^{M}
r_{g,j}
\log
\frac{r_{g,j}}{\tilde p_{g,j}}
=
D_g
\mathrm{KL}(r_g\|\tilde p_g).
\]
By nonnegativity of KL divergence,
\[
\Phi_g(Q_g)\ge 0,
\]
with equality if and only if \(r_g=\tilde p_g\). Therefore the unique minimizer
on \(\mathcal Q_g\) is
\[
Q_g=D_g\tilde p_g .
\]
\end{proof}

\section{Expansion and activation property of the rank-majorization JS penalty}
\label{app:rank_majorization_js}

We give the full construction of the rank-majorization Jensen--Shannon penalty.
For sample \(i\), let
\[
\mathcal A_i=\{j:m_{i,j}^{\mathrm{exp}}=1\},
\qquad
k_i=|\mathcal A_i|
\]
denote the feasible expert set and its cardinality. The conditional expert
allocation satisfies
\[
q_i\in\Delta(\mathcal A_i),
\qquad
q_{i,j}=0\quad\text{for }j\notin\mathcal A_i.
\]
Let
\[
q_i^{\downarrow}
=
\bigl(q_{i,(1)},\ldots,q_{i,(k_i)}\bigr)
\]
be the vector of feasible expert probabilities sorted in nonincreasing order,
so that
\[
q_{i,(1)}\ge q_{i,(2)}\ge \cdots \ge q_{i,(k_i)},
\qquad
\sum_{\ell=1}^{k_i}q_{i,(\ell)}=1.
\]
Ties may be broken arbitrarily; the construction is permutation-invariant and is
piecewise differentiable almost everywhere.

We compare this sorted allocation profile to a truncated geometric rank profile
\[
g_i
=
(g_{i,1},\ldots,g_{i,k_i})
\in \Delta^{k_i-1},
\]
defined by
\[
g_{i,\ell}
=
\frac{\exp\{-\tau_{\mathrm{rank}}(\ell-1)\}}
{\sum_{r=1}^{k_i}\exp\{-\tau_{\mathrm{rank}}(r-1)\}},
\qquad
\ell=1,\ldots,k_i,
\]
where \(\tau_{\mathrm{rank}}\ge 0\) controls the concentration of the reference
profile. Larger \(\tau_{\mathrm{rank}}\) permits a more top-heavy reference,
whereas smaller \(\tau_{\mathrm{rank}}\) encourages a flatter allocation across
available experts.

Define the cumulative rank masses
\[
R_i(t)
=
\sum_{\ell=1}^{t}q_{i,(\ell)},
\qquad
G_i(t)
=
\sum_{\ell=1}^{t}g_{i,\ell},
\qquad
t=1,\ldots,k_i .
\]
Since both \(q_i^{\downarrow}\) and \(g_i\) are probability vectors,
\[
R_i(k_i)=G_i(k_i)=1.
\]
Thus, only strict prefix inequalities \(t<k_i\) can reveal excessive
concentration.

We define the margin-violation score
\[
v_i
=
\max_{1\le t\le k_i}
\left[
R_i(t)-G_i(t)-m
\right]_{+},
\]
where \(m\ge 0\) is a slack margin and \([x]_{+}=\max\{x,0\}\). The corresponding
activation variable is
\[
\chi_i
=
\mathbf 1[v_i>0].
\]
Equivalently,
\[
\chi_i=1
\quad\Longleftrightarrow\quad
\exists\,t\in\{1,\ldots,k_i\}
\text{ such that }
R_i(t)>G_i(t)+m .
\]
Therefore, the rank penalty is activated only when the sorted router allocation
places more mass in its leading ranks than the geometric reference allows up to
margin \(m\).

For the active samples, we penalize the Jensen--Shannon divergence between the
sorted allocation profile and the geometric reference profile:
\[
\operatorname{JS}(q_i^{\downarrow}\|g_i)
=
\frac{1}{2}
\operatorname{KL}
\left(
q_i^{\downarrow}
\,\middle\|\,
s_i
\right)
+
\frac{1}{2}
\operatorname{KL}
\left(
g_i
\,\middle\|\,
s_i
\right),
\qquad
s_i=\frac{1}{2}(q_i^{\downarrow}+g_i).
\]
Expanding the KL terms gives
\[
\operatorname{JS}(q_i^{\downarrow}\|g_i)
=
\frac{1}{2}
\sum_{\ell=1}^{k_i}
q_{i,(\ell)}
\log
\frac{2q_{i,(\ell)}}{q_{i,(\ell)}+g_{i,\ell}}
+
\frac{1}{2}
\sum_{\ell=1}^{k_i}
g_{i,\ell}
\log
\frac{2g_{i,\ell}}{q_{i,(\ell)}+g_{i,\ell}} .
\]
The rank-majorization JS penalty is then
\[
\mathcal L_{\mathrm{rankJS}}(\theta)
=
\frac{1}{\sum_{i=1}^{N}d_i}
\sum_{i=1}^{N}
d_i\,
\chi_i\,
\operatorname{JS}(q_i^{\downarrow}\|g_i).
\]
Thus, the penalty acts only on samples that are softly deferred and whose
conditional expert allocation is more top-heavy than the allowed geometric
reference profile.

\begin{proposition}[Activation criterion of the rank-majorization penalty]
\label{prop:rank_activation}
For sample \(i\), suppose that
\[
R_i(t)\le G_i(t)+m
\qquad
\text{for all }t=1,\ldots,k_i .
\]
Then \(\chi_i=0\), and sample \(i\) contributes no rank-majorization JS penalty.
Conversely, if \(\chi_i=1\), then there exists a prefix rank \(t<k_i\) such that
\[
R_i(t)>G_i(t)+m,
\]
meaning that the routed expert distribution assigns excessive cumulative mass
to its top-ranked experts relative to the geometric reference profile.
\end{proposition}

\begin{proof}
By definition,
\[
v_i
=
\max_{1\le t\le k_i}
\left[
R_i(t)-G_i(t)-m
\right]_{+},
\qquad
\chi_i=\mathbf 1[v_i>0].
\]
If
\[
R_i(t)\le G_i(t)+m
\qquad
\forall t=1,\ldots,k_i,
\]
then
\[
R_i(t)-G_i(t)-m\le 0
\qquad
\forall t=1,\ldots,k_i.
\]
Therefore,
\[
\left[
R_i(t)-G_i(t)-m
\right]_{+}=0
\qquad
\forall t,
\]
and hence
\[
v_i=0.
\]
It follows immediately that
\[
\chi_i=\mathbf 1[v_i>0]=0.
\]
The sample-level contribution to the rank-majorization JS penalty is
\[
d_i\,\chi_i\,\operatorname{JS}(q_i^{\downarrow}\|g_i),
\]
which is zero whenever \(\chi_i=0\). Thus, sample \(i\) contributes no rank
penalty.

Conversely, suppose \(\chi_i=1\). Then \(v_i>0\). By the definition of \(v_i\),
there must exist at least one rank prefix \(t\in\{1,\ldots,k_i\}\) such that
\[
\left[
R_i(t)-G_i(t)-m
\right]_{+}>0.
\]
This is equivalent to
\[
R_i(t)-G_i(t)-m>0,
\]
or
\[
R_i(t)>G_i(t)+m.
\]
Because both \(q_i^{\downarrow}\) and \(g_i\) are probability vectors,
\[
R_i(k_i)=G_i(k_i)=1.
\]
For \(m\ge 0\), the inequality
\[
R_i(k_i)>G_i(k_i)+m
\]
cannot hold, since it would require
\[
1>1+m.
\]
Therefore, any activating violation must occur at a strict prefix
\(t<k_i\). Hence \(\chi_i=1\) only when the sorted router allocation places more
mass in its top \(t\) experts than the geometric reference profile allows,
up to margin \(m\). This is precisely the sense in which the penalty is activated
only by excessively top-heavy expert allocation.
\end{proof}

\section{Pseudo Labeling}
\label{app:pseudo_labeling}
We use three public glaucoma fundus-image cohorts with complementary annotation
patterns. REFUGE is a Chinese retinal fundus glaucoma benchmark with clinical
glaucoma labels and optic-disc/optic-cup annotations from seven expert ophthalmologists, but without dense per-expert diagnostic decisions \citep{orlando2020refuge}. CHAKSU is an
Indian-ethnicity glaucoma dataset acquired using multiple fundus cameras and
provides optic-disc/optic-cup contours together with binary glaucoma decisions
from five expert ophthalmologists \citep{kumar2023chakṣu}. ORIGA is a
Singaporean fundus-image dataset with image-level glaucoma labels and
optic-disc/optic-cup annotations, but without the multi-expert decision
structure needed for expert-routing supervision \citep{zhang2010origa}.
Because these cohorts differ in acquisition devices, population characteristics,
annotation protocols, and image quality, they provide a natural setting for
studying expert routing under heterogeneous supervision and distribution shift. 
\subsection{REFUGE semi-synthetic label generation}
\label{subsec:refuge_semisynthetic_labels}
For REFUGE, structural annotations are available, but dense per-expert decisions
are incomplete. We therefore generate semi-synthetic per-expert outcomes from
annotated optic-disc and optic-cup geometry. This construction is motivated by
the clinical role of optic-nerve-head morphology in glaucoma assessment:
vertical cup-to-disc ratio (vCDR) is a widely used and relatively robust index
of glaucomatous neuroretinal rim loss, and cup-to-disc measurements are commonly used in glaucoma screening and diagnosis \citep{foster2002definition, kavitha2010early, hagiwara2018computer}. We treat the generated labels as noisy
semi-synthetic supervision rather than true expert annotations.

\paragraph{Geometric biomarker extraction.}
For case \(i\) and expert/annotator \(j\), let the fitted optic-disc and
optic-cup ellipses be
\[
e^{D}_{ij}
=
(w^{D}_{ij},h^{D}_{ij},\theta^{D}_{ij},c^{D}_{x,ij},c^{D}_{y,ij}),
\qquad
e^{C}_{ij}
=
(w^{C}_{ij},h^{C}_{ij},\theta^{C}_{ij},c^{C}_{x,ij},c^{C}_{y,ij}).
\]
For an ellipse with width \(w\), height \(h\), and rotation angle \(\theta\),
let \(a=w/2\) and \(b=h/2\) denote the semi-major and semi-minor axes. We define
the vertical diameter functional
\[
\mathrm{VD}(w,h,\theta)
=
2\sqrt{
a^2\sin^2\theta
+
b^2\cos^2\theta
}.
\]
Set
\[
\mathrm{VD}^{D}_{ij}
=
\mathrm{VD}(w^{D}_{ij},h^{D}_{ij},\theta^{D}_{ij}),
\qquad
\mathrm{VD}^{C}_{ij}
=
\mathrm{VD}(w^{C}_{ij},h^{C}_{ij},\theta^{C}_{ij}).
\]
We then compute three structural biomarkers:
\[
\mathrm{vCDR}_{ij}
=
\frac{\mathrm{VD}^{C}_{ij}}{\mathrm{VD}^{D}_{ij}},
\qquad
\mathrm{aCDR}_{ij}
=
\frac{a^{C}_{ij}b^{C}_{ij}}{a^{D}_{ij}b^{D}_{ij}},
\qquad
\mathrm{dec}_{ij}
=
\frac{
\left\|
(c^{C}_{x,ij},c^{C}_{y,ij})
-
(c^{D}_{x,ij},c^{D}_{y,ij})
\right\|_2
}{
\mathrm{VD}^{D}_{ij}
}.
\]
Here, \(\mathrm{vCDR}_{ij}\) measures vertical cup enlargement,
\(\mathrm{aCDR}_{ij}\) measures relative cup area, and \(\mathrm{dec}_{ij}\)
captures cup--disc decentration. These quantities summarize clinically
meaningful optic-nerve-head structure while remaining simple enough to support
interpretable semi-synthetic label generation.

We define the geometry feature vector
\[
\phi_{ij}
=
\begin{bmatrix}
1 \\
\mathrm{vCDR}_{ij} \\
\mathrm{aCDR}_{ij} \\
\log(\mathrm{VD}^{D}_{ij}) \\
\mathrm{dec}_{ij}
\end{bmatrix}.
\]
For each expert \(j\), we fit a logistic model mapping geometry to a glaucoma
evidence score:
\begin{equation}
s_{ij}
=
w_j^\top \phi_{ij},
\qquad
p^{\mathrm{raw}}_{ij}
=
\sigma(s_{ij}),
\qquad
\sigma(t)
=
\frac{1}{1+\exp(-t)} ,
\end{equation}
where \(w_j\) is estimated by maximum likelihood on the training split
\(\mathcal D_{\mathrm{tr}}\). All fitting and calibration are performed using
training/calibration data only to avoid test-set leakage.

\paragraph{Per-expert temperature calibration.}
Because the raw logistic scores may be over- or under-confident, we calibrate
each expert-specific score using temperature scaling \citep{guo2017calibration}:
\begin{equation}
p^{\mathrm{cal}}_{ij}
=
\sigma\!\left(\frac{s_{ij}}{T_j}\right),
\qquad
T_j>0 .
\end{equation}
The temperature \(T_j\) is selected on a held-out calibration split
\(\mathcal D_{\mathrm{cal}}\) by minimizing negative log-likelihood. Calibration
quality is then evaluated on the test split \(\mathcal D_{\mathrm{te}}\).

\paragraph{Geometry-only case difficulty.}
Let \(\mathrm{vCDR}^{\mathrm{med}}_i\) denote the median vCDR across available
annotations for case \(i\). We define a geometry-only difficulty score
\begin{equation}
\beta_i
=
\exp\!\left(
-\kappa
\left|
\mathrm{vCDR}^{\mathrm{med}}_i-\tau
\right|
\right),
\qquad
\kappa>0 .
\end{equation}
The threshold \(\tau\) is chosen as the vCDR cutoff that maximizes the Youden
index on the training data, a standard criterion for selecting diagnostic
operating points from ROC analysis \citep{youden1950index, ruopp2008youden}.
Thus, anatomically borderline cases near \(\tau\) receive larger
\(\beta_i\), reflecting greater ambiguity in geometry-based diagnosis.

\paragraph{Case-specific operating point.}
Given expert-level baseline operating characteristics
\((\mathrm{Se}_j,\mathrm{Sp}_j)\), we construct case-specific sensitivity and
specificity as
\begin{align}
\operatorname{logit}(\mathrm{Se}_{ij})
&=
\operatorname{logit}(\mathrm{Se}_j)
+
\alpha_j
\left(
p^{\mathrm{cal}}_{ij}-\rho
\right)
-
b\beta_i,
\\
\operatorname{logit}(\mathrm{Sp}_{ij})
&=
\operatorname{logit}(\mathrm{Sp}_j)
-
\gamma_j
\left(
p^{\mathrm{cal}}_{ij}-\rho
\right)
-
d\beta_i,
\end{align}
where \(\alpha_j,\gamma_j\ge 0\) control how strongly calibrated structural
evidence changes the expert operating point, \(b,d\ge 0\) penalize borderline
cases, and \(\rho\) is a reference evidence level. This construction allows
expert performance to vary with case morphology while preserving each expert's
baseline sensitivity--specificity profile.

\paragraph{Poisson--binomial sampling with correctness constraints.}
Let \(c_{ij}\in\{0,1\}\) indicate whether expert \(j\) is correct on case \(i\):
\[
c_{ij}
=
\mathbf 1\{\hat y_{ij}=y_i\}.
\]
Conditioned on the true label and geometry, we model correctness as
\begin{equation}
c_{ij}
\mid
(y_i,\phi_{ij})
\sim
\mathrm{Bernoulli}(\Phi_{ij}),
\qquad
\Phi_{ij}
=
y_i\,\mathrm{Se}_{ij}
+
(1-y_i)\,\mathrm{Sp}_{ij}.
\end{equation}
The number of correct experts,
\[
K_i=\sum_{j=1}^{J}c_{ij},
\]
therefore follows a Poisson--binomial distribution because it is the sum of
independent Bernoulli variables with non-identical success probabilities
\citep{chen1997statistical}. To avoid unrealistically poor synthetic expert panels,
we condition on a minimum quality constraint \(K_i\ge k_{\min}\), yielding a
conditional Bernoulli law.

We sample exactly under this constraint using a dynamic program. Define suffix
probabilities
\[
q_j(s)
=
\Pr\!\left(
\sum_{t=j}^{J} c_{it}=s
\;\middle|\;
y_i,\phi_{i1:J}
\right),
\]
with boundary conditions \(q_{J+1}(0)=1\) and \(q_{J+1}(s>0)=0\). The recursion is
\[
q_j(s)
=
(1-\Phi_{ij})q_{j+1}(s)
+
\Phi_{ij}q_{j+1}(s-1),
\qquad
s\ge 0 .
\]
We first draw
\[
K_i
\sim
\Pr(K_i=k\mid K_i\ge k_{\min})
=
\frac{q_1(k)}{Z_i},
\qquad
k\in\{k_{\min},\ldots,J\},
\]
where
\[
Z_i
=
\sum_{u=k_{\min}}^{J}q_1(u).
\]
Given the sampled total number of correct experts, we sequentially sample
\((c_{i1},\ldots,c_{iJ})\) using
\begin{equation}
\Pr
\left(
c_{ij}=1
\;\middle|\;
\sum_{t=j}^{J}c_{it}=r
\right)
=
\frac{
\Phi_{ij}q_{j+1}(r-1)
}{
q_j(r)
}.
\end{equation}
Finally, we instantiate the semi-synthetic expert label as
\begin{equation}
\hat y_{ij}
=
\begin{cases}
y_i, & c_{ij}=1,\\
1-y_i, & c_{ij}=0,
\end{cases}
\qquad
\text{so that}
\qquad
\Pr(\hat y_{ij}=1\mid y_i,\phi_{ij})
=
y_i\Phi_{ij}
+
(1-y_i)(1-\Phi_{ij}).
\end{equation}

This procedure produces expert-specific labels that are tied to clinically
interpretable optic-disc and optic-cup morphology, preserve heterogeneous expert
operating characteristics, and enforce a weak panel-quality constraint without
assuming that all experts are equally reliable.

\subsection{ORIGA synthetic label generation}
\label{subsec:origa_synthetic_labels}

ORIGA does not provide multi-expert annotations. We therefore induce plausible
expert outcomes by retrieving morphologically similar labeled cases from REFUGE
and CHAKSU in a pretrained retinal representation space. This procedure follows
the intuition of case-based ophthalmic decision support, where visually similar
fundus images serve as reference cases for diagnosis, while the retrieved labels
are treated as noisy pseudo-supervision rather than ground-truth expert labels.

For each image, we extract embeddings using DINOv2 \cite{oquab2023dinov2} and RETFound-MAE \cite{zhou2023foundation} after
cropping the retinal region, followed by $\ell_2$ normalization so that
dot-product similarity approximates cosine similarity. For an ORIGA query image,
we restrict the retrieval pool by its ground-truth class: ORIGA images with
$y=1$ are matched only to REFUGE/CHAKSU cases with $y=1$, and ORIGA images with $y=0$ are matched only to REFUGE/CHAKSU cases with $y=0$. Within this
label-matched pool, we compute similarity between the ORIGA embedding and all
candidate embeddings, retrieve the most similar cases, fuse DINOv2 and
RETFound-MAE similarity scores after per-query normalization, and retain the
top seven nearest neighbors. The synthetic expert outcomes for the ORIGA image
are then obtained from the retrieved neighbors' expert labels,
$\hat y^{\mathrm{exp}}$, and used as pseudo-annotations.

We use cropped rather than uncropped fundus images because the cropped retinal
region better emphasizes clinically relevant morphology, including optic-disc
and cup-related structure. 
\begin{wrapfigure}[14]{r}{0.49\textwidth}
    \centering
    \includegraphics[width=0.48\textwidth]{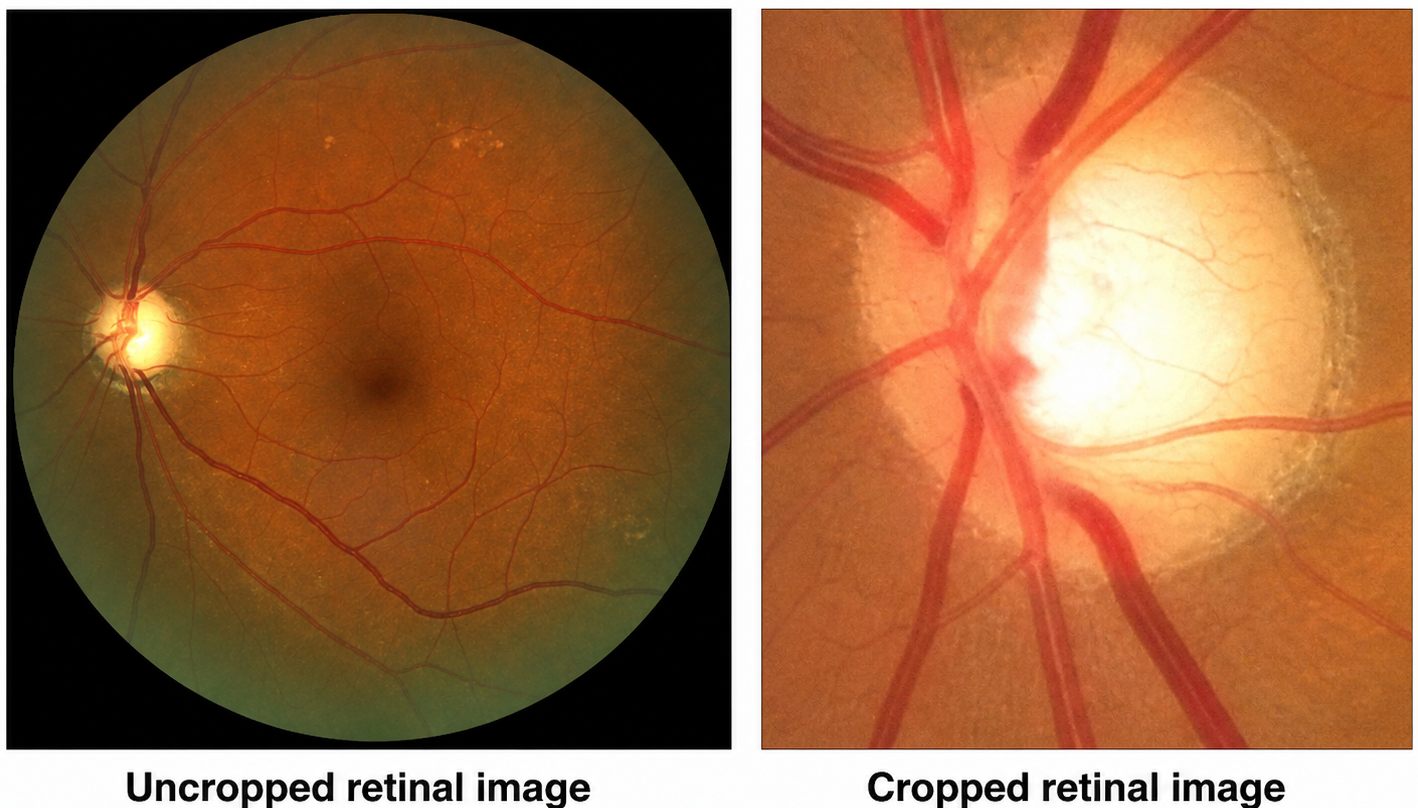}
    \caption{Retinal-region cropping used before similarity-based retrieval.}
    \label{fig:retinal}
\end{wrapfigure}
Moreover, this assumption is consistent with our
risk-stratified analysis in Figure~\ref{fig:risk_stratified_cost}, where
model-reliability risk primarily degrades the frozen AI model, whereas
structural glaucoma risk, reflected by features such as aCDR and vCDR, is a more clinically meaningful source of case difficulty. Thus, similarity search in a cropped retinal representation space is intended to retrieve reference cases
with comparable structural appearance rather than merely similar global image context.
The resulting ORIGA pseudo-labels remain sparse: the top retrieved neighbors may contain repeated labels from the same expert, and not every expert is represented
for every ORIGA case. This sparsity is desirable for our setting because it
preserves the heterogeneous and incomplete expert-availability structure that
the proposed mask-aware router is designed to handle.
\section{Experiment implementation details}
\label{sec:implementation}

This section provides the implementation details required to reproduce the main
experimental results, including the fixed data split, decision-time features,
action masks, optimization procedure, hyperparameter selection protocol,
baseline adaptations, test-time evaluation, and computational resources. All
methods are trained, selected, and evaluated under the same protocol unless
explicitly stated otherwise.

\subsection{Data splits and leakage control}
\label{subsec:data_splits}

All experiments use a fixed train/validation/test split constructed before
model comparison. The split contains 3195 retinal fundus cases from REFUGE,
CHAKSU, and ORIGA. The training split is used for fitting model parameters,
constructing train-only priors, estimating any feature-normalization statistics,
and fitting any auxiliary components required by the router. The validation
split is used only for early stopping, hyperparameter selection, pruning, and
model selection. The test split is held out until final evaluation and is not
used for hyperparameter search, prior construction, threshold selection, or
early stopping.

\begin{table}[t]
\centering
\small
\caption{
Fixed train/validation/test split used for all methods. The same split is used
for MPD$^2$-Router and every baseline. Counts are reported by cohort.
}
\label{tab:data_split}
\begin{tabular}{lrrrr}
\toprule
Split & REFUGE & CHAKSU & ORIGA & Total \\
\midrule
Train & 400 & 686 & 325 & 1411 \\
Validation & 400 & 323 & 162 & 885 \\
Test & 400 & 336 & 163 & 899 \\
\bottomrule
\end{tabular}
\end{table}

The frozen diagnostic backbone is trained only on the designated training data
and is not updated during router training. All downstream routing methods
therefore operate on the same frozen decision-time state rather than on raw
images. Any statistics used to normalize continuous routing features are
estimated on the training split and then applied unchanged to validation and
test. Similarly, the group-specific distribution priors used by MPD$^2$-Router
are constructed from the training split only and are frozen during validation
and test evaluation.

\subsection{Decision-time state and masked action space}
\label{subsec:decision_state_masks}

To ensure a fair comparison, every method receives the same decision-time
feature vector:
\[
z_i =
[
\texttt{prob\_1},
\texttt{logit\_0},
\texttt{logit\_1},
\texttt{vim\_risk\_z},
\texttt{quality\_risk},
\texttt{uncertainty},
\texttt{vCDR},
\texttt{aCDR}
]_i .
\]
These features summarize the frozen AI classifier output, model uncertainty,
OOD risk, image-quality risk, and structural glaucoma biomarkers. The router
therefore compares methods in the clinically relevant downstream decision
space, rather than giving MPD$^2$-Router privileged access to additional image
information.

The unified action space contains the AI action and $M=12$ human-expert actions.
The AI action is always feasible. Human-expert feasibility is sample-specific
and is encoded by the binary expert mask
$m_i^{\mathrm{exp}}\in\{0,1\}^{M}$. The full action mask is
\[
m_i^{\mathrm{act}} =
[1, (m_i^{\mathrm{exp}})_1, \ldots, (m_i^{\mathrm{exp}})_M].
\]
Unavailable experts are masked at both training and inference time. Thus no
method can route a case to an expert whose label is unavailable for that case.

For MPD$^2$-Router, the defer head outputs a soft deferral mass
$d_i\in(0,1)$ and the expert-allocation head outputs a conditional distribution
$q_i$ over feasible experts. The conditional expert allocation is
\begin{equation}
q_{i,j}
=
\frac{a_{i,j}s_{i,j}^{\sharp}}
{\sum_{k=1}^{M}a_{i,k}s_{i,k}^{\sharp}},
\qquad j=1,\ldots,M ,
\label{eq:conditional_expert_allocation_impl}
\end{equation}
where $a_i$ is the masked allocation distribution and $s_i^\sharp$ is the
repaired stochastic expert-support vector. The denominator is clamped by
$\varepsilon=[\text{e.g., }10^{-8}]$ for numerical stability. The full action
policy is
\[
\tilde \pi_i =
[
1-d_i,
d_i q_{i,1},
\ldots,
d_i q_{i,M}
].
\]
As a final safety layer, we apply the masked-simplex projection
\[
\Pi_{\Delta(m)}(v)
=
\frac{v\odot m}{\langle \mathbf{1},v\odot m\rangle},
\]
which is well-defined whenever at least one feasible action is present, and set
\[
\pi_i=\Pi_{\Delta(m_i^{\mathrm{act}})}(\tilde\pi_i).
\]
This guarantees that the final policy places zero probability on unavailable
actions.

\subsection{Training objective and optimization}
\label{subsec:training_objective_optimization}

MPD$^2$-Router is trained to minimize the clinical--operational deployment
objective directly, rather than a per-action correctness target. For sample
\(i\), the clinical cost of retaining the frozen AI model is
\[
C_i^{\mathrm{ai}}
=
c_{\mathrm{fn}}y_i(1-p_i^{\mathrm{ai}})
+
c_{\mathrm{fp}}(1-y_i)p_i^{\mathrm{ai}},
\]
where \(p_i^{\mathrm{ai}}\) is the frozen AI probability for glaucoma. The
clinical cost of expert \(j\) is
\[
C_{i,j}^{\mathrm{exp}}
=
c_{\mathrm{fn}}\mathbf{1}[y_i=1,\hat y_{i,j}=0]
+
c_{\mathrm{fp}}\mathbf{1}[y_i=0,\hat y_{i,j}=1].
\]
Unless otherwise stated, we use asymmetric clinical costs
\(c_{\mathrm{fn}}=2.0\) and \(c_{\mathrm{fp}}=1.5\), reflecting the higher
clinical penalty of missed glaucoma relative to false referral.

The clinical--operational routing loss is
\[
\mathcal L_{\mathrm{cost}}(\theta)
=
\frac{1}{N}
\sum_{i=1}^{N}
\left[
C_i^{\mathrm{ai}}
+
d_i
\left(
\sum_{j=1}^{M}
q_{i,j}
\left(
C_{i,j}^{\mathrm{exp}}
+
\gamma\kappa_j
\right)
-
C_i^{\mathrm{ai}}
\right)
\right],
\]
where \(d_i\) is the soft deferral mass, \(q_{i,j}\) is the conditional expert
allocation probability, \(\kappa_j\) is the expert-tier cost, and \(\gamma\)
controls the contribution of operational expert cost. The full objective is
\[
\mathcal J(\theta)
=
\mathcal L_{\mathrm{cost}}(\theta)
+
w_{\mathrm{GSDP}}\mathcal L_{\mathrm{GSDP}}(\theta)
+
w_{\mathrm{rank}}\mathcal L_{\mathrm{rank}}(\theta)
+
P_{\mathrm{AL}}(\bar d),
\qquad
\bar d=\frac{1}{N}\sum_{i=1}^{N}d_i .
\]
Here \(\mathcal L_{\mathrm{GSDP}}\) is the group-specific distribution-prior
regularizer, \(\mathcal L_{\mathrm{rank}}\) is the rank-majorization JS
regularizer, and \(P_{\mathrm{AL}}\) is the augmented-Lagrangian penalty for the
soft deferral-budget constraint.

Models are optimized with mini-batch \texttt{AdamW}. Feature normalization
statistics are computed on the training split only and reused for validation
and test. Validation performance is evaluated after each epoch, and the best
validation checkpoint after the warmup period is restored before final
evaluation. Hyperparameters are selected only on the validation split using the
adaptive fine-tuning protocol in Section~\ref{subsec:adaptive_finetuning_impl}.
The held-out test set is evaluated once after model and hyperparameter
selection; no test-set information is used for training, pruning,
hyperparameter optimization, or checkpoint selection. The implementation details
needed to reproduce the training protocol are summarized in
Table~\ref{tab:training_hparams}; exact resolved configurations are provided in
the released experiment files.
\begin{table}[t]
\centering
\small
\caption{
Training and model-selection details for MPD$^2$-Router. Hyperparameters marked
as validation-selected are chosen by the adaptive fine-tuning protocol and are
reported in the released experiment configuration files.
}
\label{tab:training_hparams}
\begin{tabular}{ll}
\toprule
Item & Setting \\
\midrule
Framework & PyTorch \\
Optimizer & \texttt{AdamW} \\
Learning rate & validation-selected \\
Weight decay & \(10^{-4}\) for router parameters \\
Batch size & 64 \\
Maximum epochs & 150 \\
Warmup epochs & validation-selected \\
Early-stopping patience & 18 validation epochs \\
Clinical costs & \(c_{\mathrm{fn}}=2.0,\ c_{\mathrm{fp}}=1.5\) \\
Prior badness costs &
\(c_{\mathrm{fn}}^{\mathrm{prior}}=1.8,\ c_{\mathrm{fp}}^{\mathrm{prior}}=1.2\) \\
Tier-cost weight \(\gamma\) & validation-selected \\
GSDP / rank--JS weights & validation-selected \\
Augmented-Lagrangian parameters & validation-selected \\
Feature normalization & train-split statistics only \\
Checkpoint selection & best validation selection score after warmup \\
Test evaluation & held-out test split, evaluated once after selection \\
Random seeds & main repeated-seed experiments use \(42\)--\(52\) \\
\bottomrule
\end{tabular}
\end{table}

\subsection{Adaptive fine-tuning and hyperparameter selection}
\label{subsec:adaptive_finetuning_impl}

We implement adaptive fine-tuning as a lightweight Optuna-based
hyperparameter optimization loop around MPD$^2$-Router training. Each trial
samples a candidate configuration, maps it to the router's optimization,
prior-regularization, and augmented-Lagrangian settings, and trains the model
with validation-based early stopping and Hyperband pruning. The held-out test
set is never used during hyperparameter search or checkpoint selection.

The search space is 16-dimensional and covers the optimizer, deferral-cost
weighting, prior construction, anti-collapse regularization, and constraint
control:
\[
\begin{aligned}
&\texttt{lr},\ \texttt{warmup\_epochs},\ \texttt{gamma\_tier},\
\texttt{tau\_bad},\ \texttt{w\_gsdp},\ \texttt{w\_rank\_js},\\
&\texttt{global\_uniform\_mix},\ \texttt{family\_uniform\_mix},\
\texttt{group\_uniform\_mix},\
\texttt{family\_n0},\ \texttt{group\_n0},\\
&\texttt{global\_mix},\ \texttt{clip\_ceiling},\
\texttt{clip\_slack},\ \texttt{al\_mu},\
\texttt{al\_lr\_lambda}.
\end{aligned}
\]
The geometric clipping anchors are not tuned separately; they are derived from
\texttt{clip\_ceiling} and \texttt{clip\_slack} through the support-size cap
\[
\texttt{clip\_max}(k)
=
\min\!\left(
\texttt{clip\_ceiling},
g_{k,1}(\rho_k)+\texttt{clip\_slack}
\right).
\]

\paragraph{Adaptive HPO protocol.}
We use a multivariate Tree-structured Parzen Estimator (TPE) sampler with
group-based conditional sampling. A hand-tuned configuration is enqueued as the
first trial to seed the search, followed by 10 random startup trials before TPE
acquisition is used. Hyperband pruning is applied with the minimum and maximum
resources tied to the warmup epoch count and the maximum training budget,
respectively, and reduction factor \(3\). Unless otherwise stated, each study
uses 80 Optuna trials.

Within each trial, pruning and early stopping are driven by the
constraint-aware validation score
\[
\mathrm{es\_base} + 10\,\mathrm{es\_violation},
\]
where \(\mathrm{es\_violation}\) measures soft deferral-budget violation. The
best validation checkpoint under this rule is restored before trial evaluation.
The outer Optuna objective uses an augmented Tchebycheff scalarization over
clinical loss, \(-\)MCC, \(-\)AUPRC, soft tier cost, and constraint violation,
relative to an EMA-updated utopia reference. We use weights
\((0.35,0.15,0.10,0.05,0.05)\) and augmentation coefficient \(\rho=0.05\).
Trials producing \texttt{NaN} losses or metrics are pruned automatically.

After the search terminates, the selected configuration is retrained with an
extended budget using the same constraint-aware validation-selection rule and
is evaluated once on the held-out test set. Despite the breadth of the search
space, the procedure is computationally lightweight: each trial takes
approximately 2.5 minutes on the compute worker.

\subsection{Baseline implementation and fairness}
\label{subsec:baseline_implementation}

All baselines are evaluated on the same fixed train/validation/test split, the
same decision-time state $z_i$, and the same masked sparse action space.
Unavailable experts are masked during training and inference whenever the
baseline formulation permits masking. When a method was originally designed for
a different action structure, we use the closest masked sparse adaptation needed
to make action-level comparison meaningful.

This benchmark is deliberately favorable to the baselines in two ways. First,
all methods route from a highly informative downstream state derived from the
frozen AI model, uncertainty/OOD pipeline, image-quality signal, and structural
glaucoma biomarkers. Second, in the benchmark configuration,
\texttt{selection\_metric} is set to \texttt{"surrogate"} while
\texttt{cost\_aware\_inference} is enabled. Therefore, several baselines retain
their original surrogate-style training objectives while also benefiting from
cost-aware routing at inference time.

The baseline adaptations are as follows. Verma-OvA is trained against
per-action correctness targets under the feasible action mask. Narasimhan-PH
fits separate predictors for AI correctness and expert correctness and then
applies cost-aware score adjustment during routing. Mao-TwoStage uses a
score-based surrogate constructed from AI correctness and expert correctness,
with optional base-cost terms. Hemmer-MoE is trained with true expert labels
inside a masked team cross-entropy over the classifier and available experts.
Keswani et al. is included as a committee-based multi-expert collaboration
baseline. Because its native formulation produces a weighted committee rather
than a single routed action, we evaluate a masked top-1 projection for
action-level comparisons.

Although the task label is a single binary label $y_i$, several baselines are
trained using derived per-action correctness patterns such as ``AI wrong,
expert 1 correct, expert 2 wrong, expert 3 unavailable.'' These targets are
substantially closer to the routing answer than ordinary binary classification.
In contrast, MPD$^2$-Router is not trained to output a per-action correctness
vector. It directly optimizes the clinical--operational routing objective and
learns when to retain the AI, when to defer, and which feasible expert to select
under asymmetric clinical cost, expert-tier cost, availability masks, and
deferral-budget constraints.

\subsection{Test-time evaluation}
\label{subsec:test_time_evaluation}

After validation-based model selection, each method is evaluated on the held-out
test split. Unless explicitly stated otherwise, soft policies are converted to
hard actions by masked argmax:
\[
\hat a_i = \arg\max_{a\in\mathcal A_i} \pi_{i,a}.
\]
If $\hat a_i$ is the AI action, the final prediction is the frozen AI prediction.
If $\hat a_i$ is a human-expert action, the final prediction is the selected
expert's label. Because all policies are projected onto the masked simplex, test
routing cannot select unavailable experts.

We report standard diagnostic metrics, including accuracy, precision, recall,
specificity, F1, and MCC. We also report clinical cost, expert cost,
total cost, deferral rate, deferral-budget violation, and expert-utilization
statistics. Expert-collapse diagnostics are computed on the human-routed subset
and include MaxShare, entropy-based effective number of experts, and normalized
Gini. Cost--performance trade-off figures use the same held-out test predictions
and the same cost definitions for all methods.

When results are averaged across random seeds, we report mean and standard
deviation over 10 random seeds. When a single run
is reported, the selected seed is fixed before test evaluation and all methods
use the same seed where applicable.

\subsection{Computational resources and runtime}
\label{subsec:compute_resources}
All routing experiments are lightweight and can be reproduced without
distributed training. Experiments are run on a single compute worker with
a 13th-generation Intel Core i9-13905H CPU (14 physical / 20 logical cores),
an NVIDIA GeForce RTX 4060 Laptop GPU (compute capability~8.9) with
8~GiB GPU memory, and 32~GiB system memory. The operating system is
Windows~11, and the main software stack is Python~3.12, PyTorch~2.10
(CUDA~12.8 build), Optuna~4.8, NumPy~2.4, and scikit-learn~1.8.

The routing models are small because they operate on an eight-dimensional
decision-time state rather than full images. Consequently, the dominant runtime
comes from repeated validation-based hyperparameter search rather than from a
single training run. Hyperband pruning further reduces wasted computation by
terminating weak configurations early. In practice, each adaptive fine-tuning
trial takes approximately 2.5 minutes, making it feasible to run many trials and
select robust configurations without large-scale compute.

\section{Experiment results}
\subsection{AI-retention safety under distribution shift}
\label{app:ai_retention}
Table~\ref{tab:app_ai_retention} shows that MPD$^2$-Router improves AI-retention safety by becoming increasingly selective under distribution shift. An effective router should retain a nontrivial AI share while keeping retained-AI accuracy high and clinical cost low. On in-distribution REFUGE, the frozen classifier is already reliable, yet MPD$^2$-Router still preserves an autonomous share of 0.755 with near-perfect retained-AI accuracy and low clinical cost. When OOD shift, AI-only inference becomes unsafe: clinical cost grows from 0.179 (REFUGE) to 0.408 (CHAKSU) and 0.773 (ORIGA). Baseline routers exhibit two opposing failure modes, retaining a large AI share at high cost or reducing retained-AI error only through aggressive deferral that leaves few cases to the AI. MPD$^2$-Router instead occupies a more Pareto-favorable region of the coverage--risk trade-off: on far-OOD ORIGA, it retains 20.2\% of cases with 100\% retained-AI accuracy and zero clinical cost, avoiding Hemmer-MoE's overly conservative 10.4\% retention. These results indicate MPD$^2$-Router's selective routing rather than uniform abstention or indiscriminate deferral.
\begin{table*}[t]
\centering
\small
\setlength{\tabcolsep}{4pt}
\renewcommand{\arraystretch}{1.1}
\caption{AI-retained subset performance under distribution shift. \emph{Share} = fraction of cases kept under autonomous AI prediction; \emph{Acc}, \emph{MCC}, and \emph{CC} (ClinicalCost) are computed on the retained subset. CC is the only minimization target ($\downarrow$); others maximize ($\uparrow$). Test set sizes: REFUGE $n{=}400$, CHAKSU $n{=}336$, ORIGA $n{=}163$.}
\label{tab:app_ai_retention}
\begin{tabular}{@{}l rrrr rrrr rrrr@{}}
\toprule
 & \multicolumn{4}{c}{REFUGE} & \multicolumn{4}{c}{CHAKSU} & \multicolumn{4}{c}{ORIGA} \\
\cmidrule(lr){2-5} \cmidrule(lr){6-9} \cmidrule(lr){10-13}
Method & Share & Acc & MCC & CC & Share & Acc & MCC & CC & Share & Acc & MCC & CC \\
\midrule
AI-Only        & 1.000 & 0.883 & 0.610 & 0.179 & 1.000 & 0.732 & 0.455 & 0.408 & 1.000 & 0.491 & 0.263 & 0.773 \\
Mao-TwoStage   & 0.780 & 0.997 & 0.962 & 0.006 & 0.530 & 0.978 & 0.343 & 0.042 & 0.288 & 0.915 & 0.753 & 0.138 \\
Narasimhan     & 0.820 & 0.997 & 0.976 & 0.006 & 0.530 & 0.983 & 0.496 & 0.034 & 0.270 & 0.909 & 0.694 & 0.159 \\
Verma-OvA      & 0.403 & 1.000 & 1.000 & 0.000 & 0.134 & 0.978 & 0.000 & 0.044 & 0.129 & 0.857 & 0.000 & 0.286 \\
Hemmer-MoE     & 0.535 & 1.000 & 1.000 & 0.000 & 0.313 & 0.990 & 0.000 & 0.019 & 0.104 & 1.000 & 0.000 & 0.000 \\
Keswani        & 0.890 & 0.961 & 0.771 & 0.077 & 0.134 & 1.000 & 0.000 & 0.000 & 0.472 & 0.818 & 0.319 & 0.364 \\
\textbf{MPD$^2$-Router} & 0.755 & 0.997 & 0.941 & 0.007 & 0.509 & 0.982 & 0.000 & 0.035 & 0.202 & 1.000 & 0.000 & 0.000 \\
\bottomrule
\end{tabular}
\end{table*}

\subsection{Cost--performance trade-off analysis}
Figures~\ref{fig:app_f1_cost_tradeoff} and~\ref{fig:app_mcc_cost_tradeoff} evaluate whether performance gains are achieved at a clinically and operationally acceptable cost. MPD$^2$-Router consistently occupies the upper-left region of the trade-off space, achieving the highest F1 and MCC while also attaining the lowest total and clinical costs among the compared learned routing methods. AI-Only has zero deferral cost but substantially worse F1/MCC and much higher clinical cost, indicating that avoiding deferral is not clinically safe.  Conversely, several baselines improve predictive performance through more expensive or less selective deferral, particularly Hemmer-MoE and Verma-OvA, which incur substantially higher total and deferral costs. These results show that MPD$^2$-Router does not merely buy accuracy through excessive referral; instead, it learns a selective routing policy that improves diagnostic performance while reducing clinical harm and maintaining a favorable operational cost profile.
\label{app:pareto_tradeoffs}
\begin{figure}[t]
    \centering
    \includegraphics[width=\textwidth]{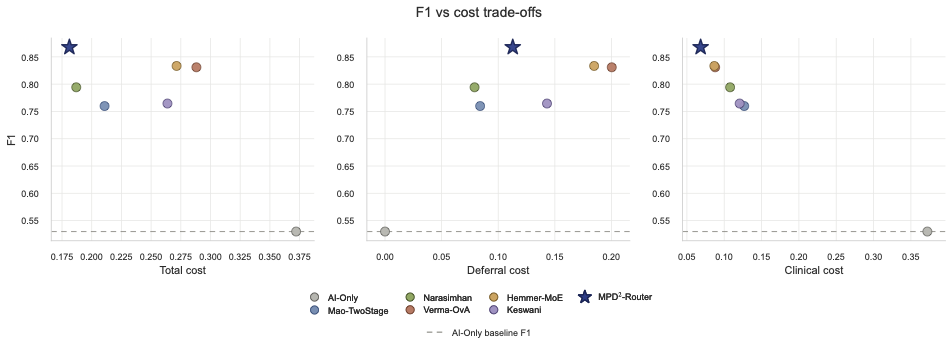}
    \caption{
    F1 versus total, deferral, and clinical cost. MPD$^2$-Router lies in the upper-left region of the trade-off space, indicating a Pareto-favorable operating point: it achieves the highest F1 while maintaining the lowest total and clinical cost among the compared methods.
    }
    \label{fig:app_f1_cost_tradeoff}
\end{figure}

\begin{figure}[t]
    \centering
    \includegraphics[width=\textwidth]{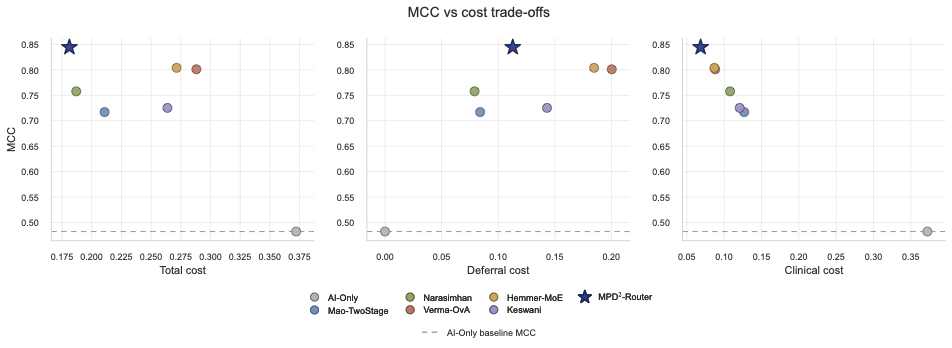}
    \caption{
    MCC versus total, deferral, and clinical cost. MPD$^2$-Router again occupies the Pareto-favorable region, achieving the strongest correlation-based classification performance without requiring the high deferral costs incurred by more defer-heavy baselines.
    }
    \label{fig:app_mcc_cost_tradeoff}
\end{figure}

\subsection{Expert performance heterogeneity and selective deferral}
\label{app:expert_performance}

\paragraph{Motivation and matched-subset protocol.}
A central premise of MPD$^2$-Router is that human readers are not interchangeable
deferral targets. This assumption is clinically realistic in glaucoma screening:
optic-disc assessment is known to exhibit substantial inter-observer variability,
with readers differing in cup-to-disc estimation, glaucomatous-damage judgment,
and sensitivity--specificity operating points
\citep{tielsch1988intraobserver,varma1992expert,abrams1994agreement, pourjavan2024evaluating}. We therefore evaluate each
expert only on the matched subset of test vcases for which that expert provides a
label. Formally, for expert \(j\), evaluation is restricted to
\(\mathcal I_j=\{i:m_{i,j}=1\}\), and MPD$^2$-Router is evaluated on the same
\(\mathcal I_j\). This controls for the heterogeneous availability mask and
ensures that the comparison measures routing quality rather than differences in
case exposure.

\begin{table}[h]
\centering
\caption{Per-expert performance on matched availability subsets. Sensitivity, specificity,
FN/FP counts, and clinical cost are reported alongside summary metrics to expose
operating-point heterogeneity that $F_1$ alone obscures. \textit{Best-oracle} is the per-case
selection of any reader producing the correct label and represents the achievable cohort
skyline. Within each cohort, rows are sorted by $F_1$.}
\label{tab:expert_table}
\small
\setlength{\tabcolsep}{4pt}
\begin{tabular}{lrrrrrrrrr}
\toprule
Expert & $n$ & Acc & Sens & Spec & $F_1$ & MCC & FN & FP & Clin.\ cost \\
\midrule
\textit{Best-oracle} & 899 & 1.000 & 1.000 & 1.000 & 1.000 & 1.000 & 0 & 0 & 0.000 \\
\midrule
\multicolumn{10}{l}{\textit{CHAKSU readers}} \\
chaksu\_expert\_1 & 415 & 0.947 & 0.921 & 0.953 & 0.864 & 0.834 &  6 &  16 & 0.087 \\
chaksu\_expert\_3 & 412 & 0.917 & 0.859 & 0.928 & 0.764 & 0.721 &  9 &  25 & 0.135 \\
chaksu\_expert\_4 & 403 & 0.881 & 0.813 & 0.894 & 0.684 & 0.625 & 12 &  36 & 0.194 \\
chaksu\_expert\_2 & 386 & 0.837 & 0.892 & 0.826 & 0.648 & 0.589 &  7 &  56 & 0.254 \\
chaksu\_expert\_5 & 382 & 0.893 & 0.456 & 0.987 & 0.602 & 0.588 & 37 &   4 & 0.209 \\
\midrule
\multicolumn{10}{l}{\textit{REFUGE readers}} \\
refuge\_expert\_1 & 452 & 0.942 & 0.793 & 0.964 & 0.780 & 0.747 & 12 &  14 & 0.100 \\
refuge\_expert\_2 & 446 & 0.926 & 0.787 & 0.942 & 0.692 & 0.657 & 10 &  23 & 0.122 \\
refuge\_expert\_4 & 471 & 0.828 & 0.797 & 0.833 & 0.557 & 0.497 & 13 &  68 & 0.272 \\
refuge\_expert\_7 & 431 & 0.889 & 0.580 & 0.929 & 0.547 & 0.485 & 21 &  27 & 0.191 \\
refuge\_expert\_5 & 484 & 0.833 & 0.873 & 0.828 & 0.542 & 0.512 &  7 &  74 & 0.258 \\
refuge\_expert\_6 & 447 & 0.810 & 0.875 & 0.801 & 0.536 & 0.496 &  7 &  78 & 0.293 \\
refuge\_expert\_3 & 443 & 0.745 & 0.745 & 0.745 & 0.383 & 0.327 & 12 & 101 & 0.396 \\
\bottomrule
\end{tabular}
\end{table}
\paragraph{Per-expert operating-point heterogeneity.}
Table~\ref{tab:expert_table} shows that expert variation is not merely a
difference in aggregate \(F_1\); it reflects distinct diagnostic styles,
tolerance to false positives versus false negatives, and implicit decision
thresholds.  This observation is consistent with prior glaucoma-reader studies.
Matched-subset \(F_1\) ranges from \(0.38\) to \(0.78\) among REFUGE
readers and from \(0.60\) to \(0.86\) among CHAKSU readers, indicating substantial
reader-quality heterogeneity. More importantly, the sensitivity--specificity
profiles reveal different clinical operating points. For example,
chaksu\_expert\_5 is highly specific (\(0.987\)) but has low sensitivity
(\(0.456\)), suggesting a conservative threshold for calling glaucoma. In
contrast, chaksu\_expert\_2 has much higher sensitivity (\(0.892\)) but lower
specificity (\(0.826\)) and many more false positives, corresponding to a more
referral-tolerant style. Similar patterns appear in REFUGE: experts~5 and~6
preserve high sensitivity but incur many false positives, whereas expert~7 is
more specificity-preserving but misses more glaucomatous cases. Thus, the
available readers differ not only in accuracy but also in the types of diagnostic
errors they tolerate.

This heterogeneity motivates the design of MPD$^2$-Router. A single pooled
``human expert,'' majority vote, or unstructured defer action cannot represent
these reader-specific operating points. The best-oracle row in
Table~\ref{tab:expert_table} is not a deployable method, but it is informative:
reader errors are not perfectly aligned, so the most useful human reviewer is
case-dependent. MPD$^2$-Router directly targets this setting by separating
whether to defer, through \(d_i\), from whom to defer to, through the conditional
masked allocation \(q_i\). The router can therefore learn when human review has
positive marginal value and which available expert is most appropriate for the
case, rather than treating all experts as exchangeable.

\paragraph{Matched comparison with MPD$^2$-Router.}
Figure~\ref{fig:barbell cahrt} compares each individual reader with
MPD$^2$-Router on exactly the same matched cases. MPD$^2$-Router improves
matched-subset \(F_1\) in all \(12/12\) expert pairings, with gains as large as
\(+0.48\) over the weakest reader and a positive gain even over the strongest
reader. The smaller gap near the best CHAKSU reader reflects a natural ceiling
effect, but the dominance remains strict. This shows that MPD$^2$-Router is not
simply compensating for weak annotators; it adds complementary, risk-aware
routing signal even when the available human prior is strong.

\begin{figure}[t]
    \centering
    \includegraphics[width=\linewidth]{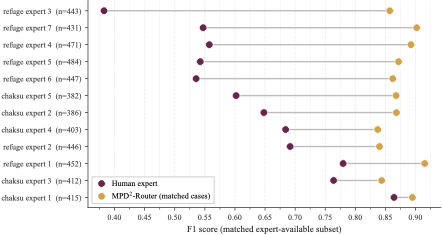}
    \caption{
    Matched-subset \(F_1\) comparison between individual human experts and
    MPD$^2$-Router. Each row evaluates one expert and MPD$^2$-Router on the same
    cases where that expert is available; \(n\) denotes the matched availability
    size, and gray lines show the \(F_1\) gap.
    }
    \label{fig:barbell cahrt}
\end{figure}

\paragraph{Spatial evidence for selective deferral.}
Figure~\ref{fig:spatial_performance_map} resolves the same comparison over a
two-dimensional embedding of the test distribution. The mean available-human
accuracy appears high across much of the manifold, but this aggregate view masks
the per-expert heterogeneity quantified in Table~\ref{tab:expert_table}. In
contrast, the frozen AI classifier exhibits a localized failure region where
accuracy collapses. MPD$^2$-Router largely recovers this region, leaving residual
errors sparse rather than spatially clustered. Its deferral mass is concentrated
on the AI failure region and suppressed where the frozen classifier is already
reliable. Thus, MPD$^2$-Router does not spend expert capacity uniformly or merely
to satisfy a budget; it makes selective, Pareto-favorable human--AI routing
decisions by allocating expert review where the expected marginal clinical value
is highest.

\begin{figure}[!htbp]
    \centering
    \includegraphics[
        width=\textwidth,
        trim=10 0 10 30,
        clip
    ]{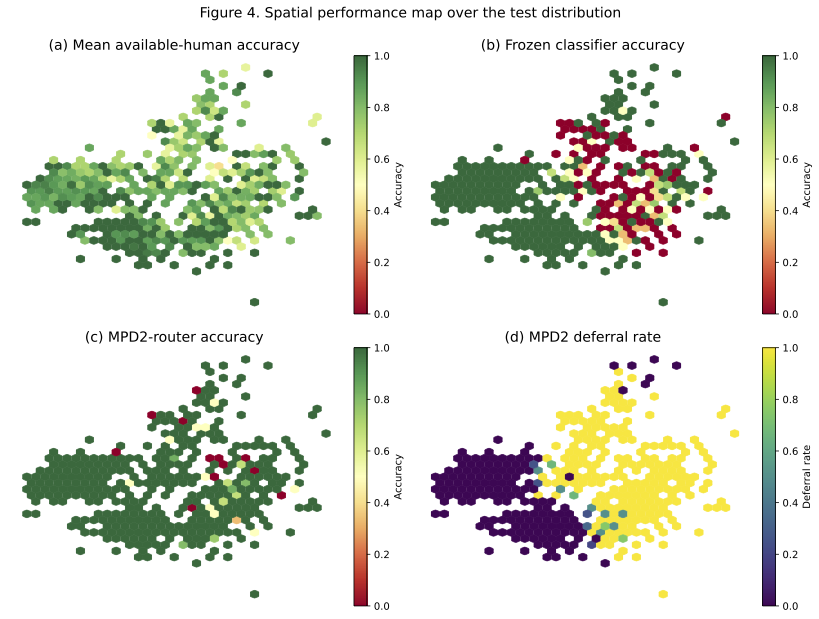}
    \caption{
    Spatial performance map over the test distribution. MPD$^2$-Router
    concentrates deferral on localized AI failure regions while preserving AI
    decisions in regions where the frozen classifier remains reliable.
    }
    \label{fig:spatial_performance_map}
\end{figure}

Together, the matched-subset, operating-point, and spatial analyses support the
central claim of MPD$^2$-Router: glaucoma deferral is not a binary choice between
AI and a homogeneous human fallback. It is an instance-wise, mask-aware allocation
problem under heterogeneous reader behavior, asymmetric diagnostic harm, and
localized model failure. MPD$^2$-Router improves over every individual expert on
that expert's own available cases precisely because it learns this structured
human--AI complementarity.

\section{Dataset details.}
\label{app:data}
\paragraph{Dataset access, licenses, and terms of use.}
We use only previously released ophthalmic datasets and credit the original
creators in all cases. Raw fundus images and raw annotations are not
redistributed with this submission. Instead, we provide preprocessing scripts,
configuration files, split identifiers, and instructions for obtaining each
dataset from its original source, so that users can reproduce the experiments
after agreeing to the corresponding dataset terms. All datasets were used only
for non-commercial academic research. No web-scraped medical images are used.

\begin{table}[t]
\centering
\small
\caption{Public ophthalmic datasets and asset-use information. We do not
redistribute raw medical images or annotations; users must obtain each dataset
from the original source and follow the corresponding license or access terms.}
\label{tab:dataset_assets}
\begin{tabular}{p{0.13\linewidth}p{0.25\linewidth}p{0.27\linewidth}p{0.25\linewidth}}
\toprule
Dataset & Source and credit & Annotations used in this work & Access, license, and redistribution \\
\midrule
REFUGE &
REFUGE Retinal Fundus Glaucoma Challenge \citep{orlando2020refuge};
official challenge page: \url{https://refuge.grand-challenge.org/};
Figshare DOI: \url{https://doi.org/10.6084/m9.figshare.26049574}. &
Color fundus images, clinical glaucoma labels, and optic-disc/optic-cup
annotations. REFUGE provides 1200 images with official train/validation/test
subsets; in our work it is also used to train the frozen AI backbone and to
derive structural features for routing. &
Released on Figshare under
Creative Commons Attribu-
tion 4.0 International (CC
BY 4.0). 
Available from the original challenge/data-hosting platform under its access
terms. We do not redistribute raw REFUGE images or annotations. \\

CHAKSU &
Chákṣu IMAGE dataset \citep{kumar2023chakṣu};
Figshare DOI: \url{https://doi.org/10.6084/m9.figshare.20123135}. &
1345 Indian-ethnicity fundus images acquired using three fundus-camera types,
with optic-disc/optic-cup contours and binary glaucoma decisions from five
expert ophthalmologists. We use the per-expert decisions and annotations for
multi-expert routing supervision. &
Released on Figshare under Creative Commons Attribution 4.0 International (CC BY 4.0). We cite the dataset and do not redistribute raw images in our
supplementary material. \\

ORIGA-light &
ORIGA-light dataset \citep{zhang2010origa}. 
Figshare DOI: \url{https://doi.org/10.6084/m9.figshare.20123135}&

650 Singaporean fundus images with image-level glaucoma labels and
optic-disc/optic-cup annotations. ORIGA does not provide the dense
multi-expert decision structure required for direct expert-allocation supervision, so we use it as an expert-sparse cross-cohort evaluation setting. &
Released on Figshare under Creative Commons Attribution 4.0 International (CC BY 4.0) and available through the original data-owner access/request procedure or the
specific public repository used by the researcher. We do not claim ownership or redistribute raw ORIGA images.  \\
\bottomrule
\end{tabular}
\end{table}

\paragraph{Preprocessing and reproducible splits.}
For each dataset, we preserve the original image identifiers and construct a
single fixed split used throughout all experiments. REFUGE follows the official
challenge split when training and validating the frozen AI model. CHAKSU follows
the released train/test organization where applicable. For ORIGA, which does
not provide dense multi-expert routing labels, we use a fixed stratified split
with the same random seed across all methods. 

\paragraph{Use of existing code and model assets.}
All third-party software libraries and pretrained backbones used for feature
extraction or initialization are cited in the main paper or implementation
appendix, together with their versions and licenses where available. The
proposed MPD$^2$-Router training code, prior construction, mask-aware routing
logic, and evaluation scripts are our own implementation. No external clinical
decision model is redistributed as part of this submission unless explicitly
permitted by its license.

\clearpage

\newpage
\end{document}